\documentclass[10pt]{article} 
\usepackage[accepted]{tmlr}

\usepackage{amsmath,amsfonts,bm}

\def\eqref#1{equation~\ref{#1}}

\def\1{\bm{1}}

\DeclareMathAlphabet{\mathsfit}{\encodingdefault}{\sfdefault}{m}{sl}
\SetMathAlphabet{\mathsfit}{bold}{\encodingdefault}{\sfdefault}{bx}{n}

\usepackage{microtype}
\usepackage{graphicx}
\usepackage{subfigure}
\usepackage{booktabs} 

\usepackage{longtable}
\usepackage[algoruled, linesnumbered]{algorithm2e} 

\usepackage[utf8]{inputenc} 
\usepackage[T1]{fontenc}    
\usepackage{hyperref}       
\usepackage{url}            
\usepackage{booktabs}       
\usepackage{amsfonts}       
\usepackage{nicefrac}       
\usepackage{microtype}      
\usepackage{xcolor}         

\usepackage{amsmath}
\usepackage{amssymb}
\usepackage{mathtools}
\usepackage{amsthm}

\usepackage{xspace}
\usepackage{booktabs}
\usepackage{makecell}
\usepackage{multirow}
\usepackage{multicol}
\usepackage{tabularx}
\usepackage[most]{tcolorbox}
\usepackage[para]{threeparttable}

\usepackage{xkeyval}

\usepackage{wrapfig}

\usepackage{mdwlist} 
\usepackage{paralist}

\usepackage[capitalize,noabbrev]{cleveref}

\title{Birds of a Feather Trust Together: Knowing When to Trust a Classifier via Adaptive Neighborhood Aggregation}

\author{\name Miao Xiong\textsuperscript{1}, Shen Li\textsuperscript{1}, Wenjie Feng\textsuperscript{1}, Ailin Deng\textsuperscript{2}, Jihai Zhang\textsuperscript{2}, Bryan Hooi\textsuperscript{1,2} \\
        \email \{miao.xiong, shen.li\}@u.nus.edu, wenjie.feng@nus.edu.sg, ailin@u.nus.edu, \{jihai, bhooi\}@comp.nus.edu.sg \\
      \addr \textsuperscript{1} Institute of Data Science, National University of Singapore \\
      \textsuperscript{2} Department of Computer Science, National University of Singapore 
}

\def\month{08}  
\def\year{2022} 
\def\openreview{\url{https://openreview.net/forum?id=p5V8P2J61u}} 

\newtheorem{theorem}{Theorem}

\newtheorem{problem}{Problem}

\theoremstyle{definition}

\theoremstyle{remark}

\newcommand{\hide}[1]{}

\definecolor{babyblueeyes}{rgb}{0.19, 0.55, 0.91}

\newcommand{\method}{\textsc{NeighborAgg}\xspace}  
\newcommand{\methodcmd}{\textsc{NeighborAgg-CMD}\xspace}  

\newcommand{\predonly}{\textit{ProbOnly}\xspace}
\newcommand{\neionly}{\textit{NeighOnly}\xspace}

\newcommand{\kdtree}{KD-tree\xspace}
\newcommand{\trustscore}{trustworthiness score\xspace}

\newcommand{\aggoperator}{\textsc{Agg}\xspace}

\newcommand{\linecomment}[1]{\textcolor{babyblueeyes}{\textit{$\triangleright$ #1}}}

\begin{document}

\maketitle

\begin{abstract}
How do we know when the predictions made by a classifier can be trusted? This is a fundamental problem that also has immense practical applicability, especially in safety-critical areas such as medicine and autonomous driving. 
The de facto approach of using the classifier's softmax outputs as a proxy for trustworthiness suffers from the over-confidence issue; while the most recent works incur problems such as additional retraining cost and accuracy versus trustworthiness trade-off.
In this work, we argue that the trustworthiness of a classifier's prediction for a sample is highly associated with two factors: the sample's neighborhood information and the classifier's output. 
To combine the best of both worlds, we design a model-agnostic post-hoc approach \method to leverage the two essential information via an adaptive neighborhood aggregation. 
Theoretically, we show that \method is a generalized version of a one-hop graph convolutional network, inheriting the powerful modeling ability to capture the varying similarity between samples within each class. 
We also extend our approach to the closely related task of mislabel detection and provide a theoretical coverage guarantee to bound the false negative.
Empirically, extensive experiments on image and tabular benchmarks verify our theory and suggest that \method outperforms other methods, achieving state-of-the-art trustworthiness performance. \footnote{Our code is publicly available at https://github.com/MiaoXiong2320/NeighborAgg.git.}.

\end{abstract}

\section{Introduction}
\label{sec:intro}

In recent years, interactions with AI systems have become increasingly pervasive in all walks of our daily lives. 
As machine learning models become more widely involved in our decision-making processes, 
the robustness and trustworthiness of their decisions need to be carefully scrutinized~\citep{Varshney2017OnTS}. 
This is of vital importance in many scenarios, 
especially in safety-critical areas, such as medical applications,
where successful deployment is highly dependent on a model's ability to detect an incorrect prediction, 
so that humans can intervene when necessary~\citep{shi2019probabilistic, chang2020data, li2021spherical}.
This leads to our central question: \emph{how can we know when the predictions made by a classifier can be trusted?}

In this paper, we investigate the trustworthiness of the prediction given by a classifier, which serves as a measure for the  classifier's quality rather than the data. 
Concretely, given some i.i.d data and an pretrained classifier
(referred to as \emph{`base classifier'} hereinafter),
the goal is to devise a discriminative and accurate \emph{trustworthiness score}, such that \textbf{higher scores indicate a belief that the classifier's predicted class is more likely to be correct}. 
In this way, the users can easily determine whether they should trust the prediction output by machine learning models, or they should resort to domain experts for manual predictions. In the literature, this task is also referred to as ``trustworthiness prediction'', ``failure prediction'', or ``misclassification detection''~\citep{jiang2018trust, corbiere2019addressing}.

\begin{figure}[!t]
    \small
    \centering    
    \includegraphics[width=0.99\linewidth]{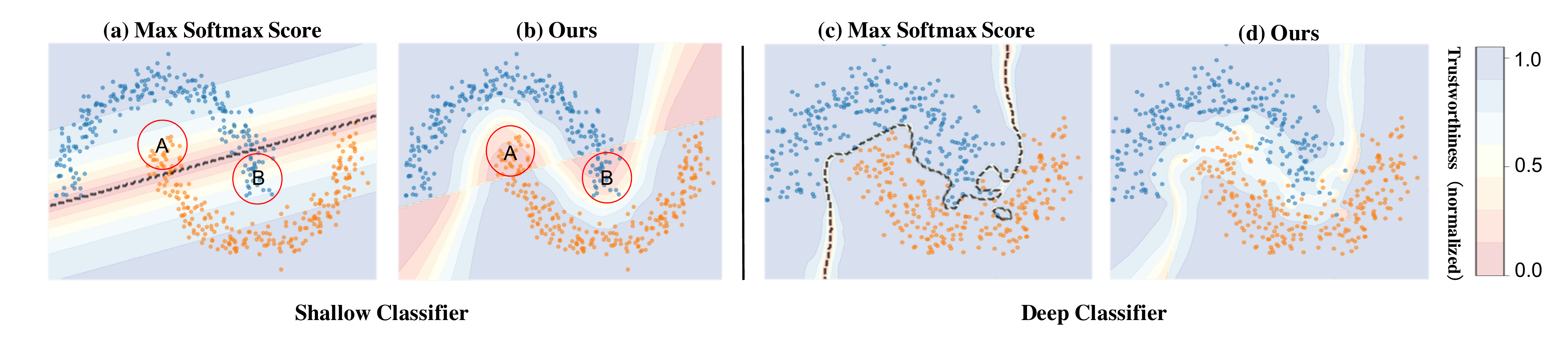}
    \caption{Comparison between max softmax scores and \method scores (ours), 
    based on a shallow base classifier (left) and a deep base classifier (right), respectively. 
    The color of the points indicates its ground truth label while the color of the background shows the corresponding trustworthiness score. The dotted black line demonstrates the decision boundary of the base classifier. 
    (\textbf{a}): The points marked by circles `A' and `B' overstep the decision boundary, being misclassified while some of them still get high max softmax scores. On the contrary, (\textbf{b}): our algorithm can correctly assign these points the lowest trustworthiness scores. 
    (\textbf{c}): Max softmax scores from base classifiers are potentially over-confident near the decision boundary whereas (\textbf{d}): our proposed score resolves this issue in a model-agnostic manner by inspecting their neighbors.
    } 
    \label{fig:toy-vis} 
\end{figure}

The most common approach is to employ a classifier's softmax output (i.e. the maximal value of a softmax vector, referred to as \textit{confidence score} hereinafter) as the proxy for trustworthiness~\citep{hendrycks17baseline}. However, this approach has been found to be over-confident~\citep{guo2017calibration}. 
\figurename{~\ref{fig:toy-vis}} illustrates this issue 
over 2D toy datasets:
the points marked by circles `A' and `B' in \figurename{~\ref{fig:toy-vis}}a are misclassified with high confidence scores.
In \figurename{~\ref{fig:toy-vis}}c, while all samples have been correctly classified, the base classifier assigns excessively high confidence scores on almost all the data points, even those near the classification boundary, making the decision boundary (full of bends and curves) prone to noise.
On the contrary, our method addresses these issues by utilizing information from the \emph{neighbors} of each point rather than just the point itself, thereby giving much lower trustworthiness scores to those misclassified points (\figurename{~\ref{fig:toy-vis}}b) 
and better reflecting the uncertainty of points near the decision boundary (\figurename{~\ref{fig:toy-vis}}d). 

\begin{wrapfigure}{r}{0pt}
    \centering
    \includegraphics[width=0.25\linewidth]{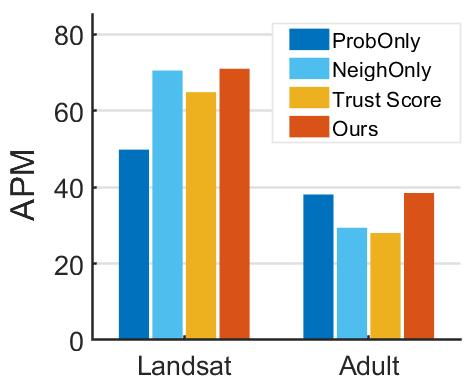}
    \vspace{-0.3in}
    \caption{Sources of information for trustworthiness prediction. Our adaptive approach outperforms other methods that only use the classifier's prediction (\texttt{ProbOnly}) or neighborhood information (\texttt{NeighOnly}), and \texttt{Trust Score}.}
    \label{fig:perform_example}
\end{wrapfigure}

Other related works include uncertainty-aware methods and post-hoc methods. The uncertainty-aware methods such as MC-dropout~\citep{Gal2016DropoutAA}, Deep Ensemble~\citep{Lakshminarayanan2017SimpleAS} and Dirichlet-based approaches \citep{Charpentier2020PosteriorNU} typically involve retraining the classifier due to the modification of network architecture, and can incur trade-offs between classifier accuracy and the performance of trustworthiness prediction.
In contrast, the \emph{post-hoc} setting avoids such extra-cost by focusing on a pretrained classifier. Among these  algorithms, \citet{corbiere2019addressing} builds on a strong assumption that the base classifier is always over-confident,
which fails in many cases~\citep{Wang2021BeCT}.
Trust Score \citep{jiang2018trust} leverages the neighborhood information by a hand-designed and non-trainable function, suffering from limited functional space and modeling capacity. 




Inspired by the commonly-held \emph{neighborhood-homophily} assumption~\citep{ fix1989discriminatory}, 
we argue that \emph{the trustworthiness of a classifier's prediction for a given sample
is highly associated with 
the sample's neighborhood information}, such as their labels and distances to the point itself.
That is, if a sample's predicted label is consistent with the majority of its neighbors' labels, this prediction is more likely to be reliable; otherwise, we tend to assign it a lower trustworthiness score.
To capture the various correlation between the sample and its neighborhood in a more flexible manner,  we devise an adaptive approach to learn the scoring function, thereby ensuring superior capacity than Trust Score~\citep{jiang2018trust}.
\figureautorefname{~\ref{fig:perform_example}} verifies the advantage by showing that the adaptive function (\texttt{NeighOnly} and \texttt{Ours}) outperforms \texttt{Trust Score}. 

Furthermore, we believe that the classifier's predictive output is also an indispensable source of 
information for the trustworthiness prediction, 
if not more so than the sample's neighborhood information, 
particularly in cases where the classifier is sufficiently reliable or the neighborhood-homophily assumption does not perfectly meet.
This is further borne out by  \figureautorefname{~\ref{fig:perform_example}} 
where using the classifier output (\texttt{ProbOnly} and \texttt{Ours}) outperforms 
using only neighborhood information (\texttt{NeighOnly} and \texttt{Trust Score}) for the Adult dataset. 
In this paper, we propose a model-agnostic algorithm, termed as \method, 
for the trustworthiness prediction by leveraging the neighborhood information and the classifier output via an adaptive scoring function that combines the best of both worlds. 
Theoretically, we demonstrate that our method is essentially a generalized one-hop graph convolutional network, 
and hence inherits the powerful modeling capacity to capture the varying similarity within each class, making it insensitive to hyperparameters for neighbor selection.
Owing to the adaptive design, 
these two factors are able to act in a complementary manner when determining the \trustscore.
Our method is also effective by achieving $7.63\%$ gain on APM and $2\%$ gain on AUC on average for the tabular dataset.



Additionally, we apply our approach to the closely related task of detecting mislabeled data samples, and propose the \methodcmd algorithm for mislabel detection.  
Furthermore, we obtain a theoretical coverage guarantee for this algorithm to bound the probability of false negative predictions.
To the best of our knowledge, the present work is 
the first to adapt to real-world noisy data setting and achieves a promising result, 
which we believe is of independent interest.

In summary, our main contributions are as follows:

\begin{compactitem}
    \item We propose a model-agnostic post-hoc algorithm \method to measure the trustworthiness of a classifier's predictions. Moreover, by demonstrating the theoretical equivalence with a generalized graph convolutional network, we provide a better understanding into how our approach works. 
    \item We propose \methodcmd, which adapts our method to mislabel detection and provide a noise-robust coverage guarantee to bound the false negative probability.
    \item Experiments on multiple tabular and image datasets showcase that the proposed \method consistently outperforms other state-of-the-art methods by clear margins.
    Additionally,
    we show that \methodcmd is able to identify mislabelled samples with promising results. 
\end{compactitem}


\section{Preliminaries and Notations}
\label{sec:notation}

We aim to measure a classifier's trustworthiness in the context of multi-class classification with $C \ge 2$ categories.\
Given a set of $N$ data points $\mathcal{X} = \left\{\mathbf{x}^{(1)}, \ldots, \mathbf{x}^{(N)}\right\}$, 
with $\mathbf{x}^{(i)} \in$ $\mathbb{R}^{D}$,  and their corresponding labels 
$\mathcal{Y} = \left\{ y^{(1)}, \ldots, y^{(N)}\right\}$, with $y^{(i)} \in \mathcal{C} = \left\{0,1,\dots,C-1 \right\}$, 
let bold $\mathbf{y}^{(i)} \in \mathbb{R}^{C}$ denote the one-hot encoding of $y^{(i)}$.  
The dataset is split into training, validation and test set, denoted by 
$(\mathcal{X}_{tr}, \mathcal{Y}_{tr}), (\mathcal{X}_{val}, \mathcal{Y}_{val}), (\mathcal{X}_{ts}, \mathcal{Y}_{ts})$, respectively.
Formally, we define the base classifier as a mapping 
$\mathcal{F}: \mathcal{X} \mapsto \mathbb{R}^{C}$
which takes a data point $\mathbf{x}$ as input and 
outputs its predicted probability vector 
or logits $\mathbf{p} \in \mathbb{R}^{C}$ and predicted class $\hat{y}$. 
Unless otherwise stated, vectors and matrices are denoted by boldface lowercase and uppercase letters, respectively, and sets are denoted by calligraphic letters. 
All vectors are treated as \emph{column} vectors throughout the paper.

\begin{problem}[Trustworthiness Prediction]
Given a base classifier $\mathcal{F}: \mathcal{X} \mapsto \mathbb{R}^{C}$, 
the trustworthiness prediction problem is to 
give a trustworthiness score $t(x)= \mathcal{T}(x; \mathcal{F}, \mathcal{X}_{tr}, \mathcal{Y}_{tr}) \in \mathbb{R}$ for any $x \in \mathcal{X}_{ts}$\footnote{$\{\mathcal{X}_{ts}, \mathcal{Y}_{ts}\}$ and $\{\mathcal{X}, \mathcal{Y}\}$ are i.i.d datasets.}, where $\mathcal{T}$ is the designed function for trustworthiness prediction, with the goal that in perfect condition:
\begin{equation*}
t(x)=\left\{\begin{array}{ll}
0 & \quad y \neq \hat{y} \\
1 & \quad y=\hat{y}.
\end{array} \right.
\end{equation*}

\end{problem}

\section{Proposed Method}
\label{sec:meth}

\begin{figure}[t]
    \centering    
 \includegraphics[width=0.8\linewidth]{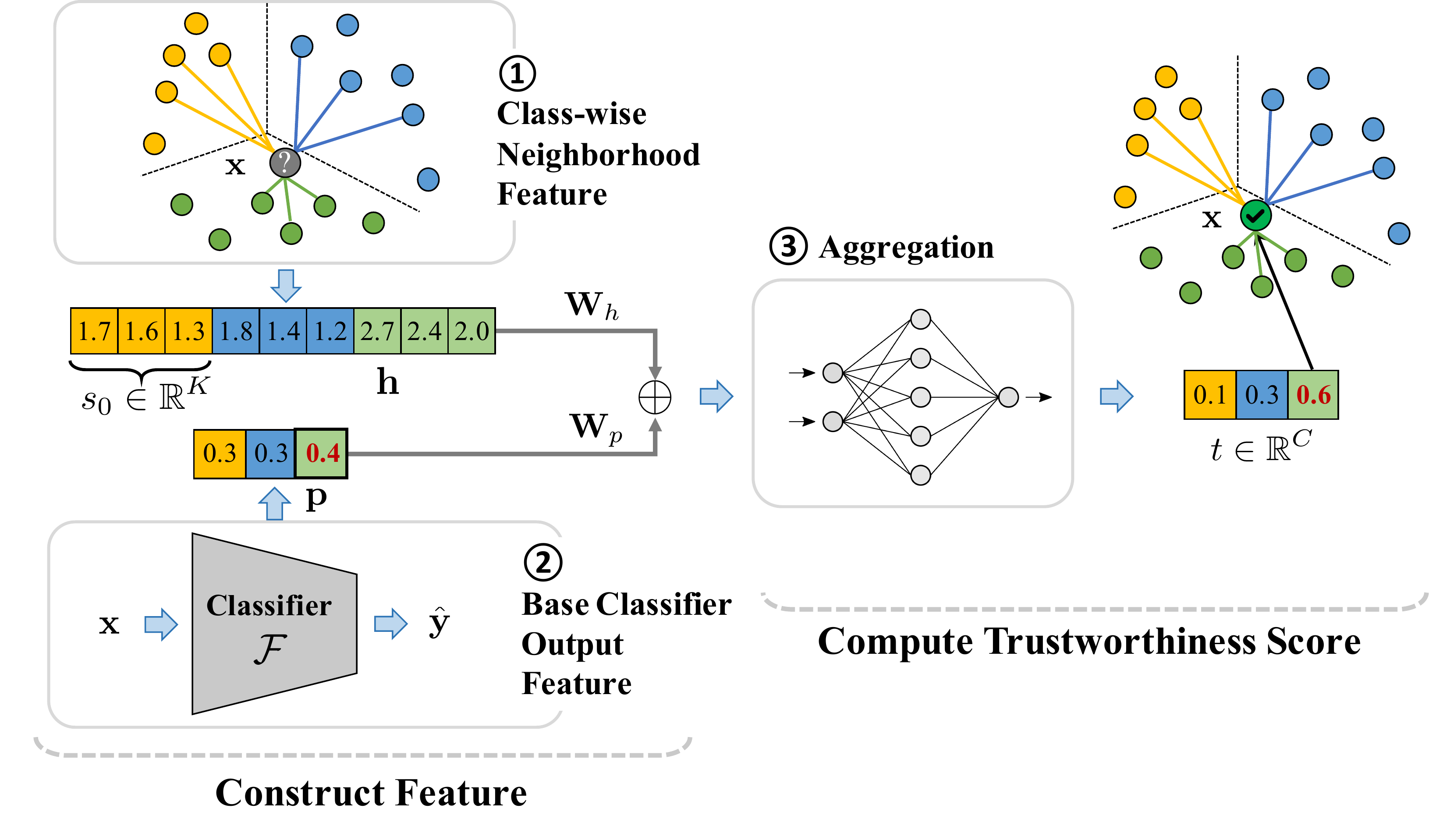}
    \vspace{-0.1in}
    \caption{A conceptual illustration of the proposed \method with $K=3$ and $C=3$. Given a sample $\mathbf{x}$ in question, we first compute two features: (1) class-wise neighborhood feature $\mathbf{h} = [\mathbf{s}_0, \mathbf{s}_1, \mathbf{s}_2]$ that reflects the similarity to its $K$ neighbors across every class, and (2) the base classifier's probability vector $\mathbf{p}$. Then, after two linear transforms $\mathbf{W_h}, \mathbf{W_p}$ followed by concatenation, (3) \method aggregates the information and outputs the final \trustscore corresponding to the base classifier's predicted class. Compared to the confidence score $0.4$, our approach assigns more trustworthiness to the predicted class by increasing the score to $0.6$.
    \label{fig:algo_image}
    }
\end{figure}

In this section, we first introduce the training and inference algorithms of our proposed \method. Then, we theoretically show that \method is a generalized version of a one-hop graph convolutional network. Lastly, we show a promising extension of \method applied in mislabel detection with a theoretical coverage guarantee.
\subsection{Algorithm}

As stated in the introduction, 
one of our key observations is that the trustworthiness of the classifier's prediction for a sample is highly associated with two information sources: 
\emph{the neighborhood of the sample} and \emph{the predictive output of the classifier}.
These two components can interact in a variety of ways.
How can we utilize the two information to determine a more reliable trustworthiness score?

Next, we will introduce our proposed model termed \emph{\method} and elaborate on how these two components are constructed and efficiently aggregated by our method.
The overall framework is illustrated in \figurename{~\ref{fig:algo_image}}.



\paragraph{Feature Construction.}
For a given sample $\mathbf{x} \in \mathcal{X}$, 
we utilize two input features: 
the neighborhood vector $\mathbf{h}$ 
and the classifier output vector $\mathbf{p}$ as shown in \figureautorefname{~\ref{fig:algo_image}}. %

For the classifier output feature, we use the aforementioned vector from the classifier $\mathbf{p} = \mathcal{F}(\mathbf{x})$.

The neighborhood vector consists of the similarity  
of a sample to its $K$ nearest neighbors across all the $C$ classes in the training dataset.
Specifically, for the sample $\mathbf{x}$ and each class $c \in \mathcal{C}$, 
we find its $K$ nearest neighbors from class $c$ of the training dataset: 
$\mathcal{N}_c= \{ \mathbf{n}_{c1}, \cdots, \mathbf{n}_{cK} \}$ 
and construct a similarity vector $\mathbf{s}_c$ as
\begin{equation}
    \label{eq:d}
    \mathbf{s}_c = \left[s_{c1}, s_{c2},\dots ,s_{cK} \right]^T,
          s_{ck} = \mathcal{K}_{f}(\mathbf{x},\mathbf{n}_{ck}), 
\end{equation}
where $\mathcal{K}_{f}$ is the Laplacian kernel with a transform $f$, 
i.e. $\mathcal{K}_{f}({\mathbf{x}}, \mathbf{z}) = \exp \left(-\left\Vert f(\mathbf{x})-f(\mathbf{z})\right\Vert_2\right)$. Cosine similarity can be used as well; but we find that in practice, Euclidean distance wrapped into Laplacian kernel performs better. For tabular dataset, the transform $f$ is set to the identity mapping of the original data, which we find empirically yields sufficiently good performance. For image dataset, a more complex transform (e.g. the backbone of the base classifier) are used for higher performance.
Then, the final neighborhood vector $\mathbf{h}$ is constructed 
by concatenating all such class-wise similarity vectors: 
\begin{equation}
    \label{equ:neighborvec}
    \mathbf{h} = [\mathbf{s}_1 \mathbin\Vert \mathbf{s}_2\mathbin\Vert,\cdots,\mathbin\Vert\mathbf{s}_C ],    
\end{equation}
where $\cdot\mathbin\Vert\cdot$ denotes the column concatenation operator. The procedure is shown in the \figureautorefname{~\ref{fig:algo_image}}.




\begin{algorithm}[t]
    \small
    \SetAlgoLined
    \KwIn{Training set $({\mathcal{X}_{tr},\mathcal{Y}_{tr}})$; Validation set $({\mathcal{X}_{val},\mathcal{Y}_{val}})$; Base classifier $\mathcal{F}$; Kernel $\mathcal{K}_{f}$; Aggregator 
\aggoperator; Training epoches $M$; Number of neighbors $K$.
}
    
    \KwOut{Parameters of \method: $\mathbf{W}_h$, $\mathbf{W}_p$, $\mathbf{W}$ (parameters of \aggoperator).}
    
     Initialize $\mathbf{W}_h$, $\mathbf{W}_p$, $\mathbf{W}$; 
    
    \For{$c=1$ to $C$ }{ \linecomment{Building class-wise KD-trees using the training set}
    
        Split from ${\mathcal{X}_{tr}}$: $ \mathcal{X}_c=\{\mathbf{x}|\mathbf{x}\in \mathcal{X}_{tr}, y=c \}$;
        
        Construct a \kdtree $\text{KDT}_c$ using $\mathcal{X}_c$ based on the kernel $\mathcal{K}_{f}$;
    }
    \For{$\text{epoch}=1$ to $M$}{ \linecomment{Training our \method using the validation set}
    
    \For{$\mathbf{x}$ in $\mathcal{X}_{val}$}{
        \For{$c=1$ to $C$}{
            Find $K$ nearest neighbors of $\mathbf{x}$ from $\text{KDT}_c$;
            
            Compute similarity vector $\mathbf{s}_c$ using \equationautorefname{~(\ref{eq:d})};
        }
        Compute the neighborhood vector $\mathbf{h} = [\mathbf{s}_1 \mathbin\Vert,\cdots,\mathbin\Vert\mathbf{s}_C ]$; 
        
        Compute the predicted vector with $\mathcal{F}$: $\mathbf{p} = \mathcal{F}(\mathbf{x})$;
        
        Compute the trustworthiness $\mathbf{t}$ using Aggregator: $\mathbf{t}=\aggoperator\left(\mathbf{W}_h \mathbf{h}, \mathbf{W}_p \mathbf{p}\right)$; 
        
        Compute the loss function using \equationautorefname{~(\ref{eq:loss})};
        
        Update $\mathbf{W}_h$, $\mathbf{W}_p$, $\mathbf{W}$ via gradient descent;

    }
    
    } 
    \Return $\mathbf{W}_h$, $\mathbf{W}_p$, $\mathbf{W}$
    \caption{\small Training Algorithm of \method}
    \label{alg:train}

\end{algorithm}

\paragraph{Aggregation.}
\label{sec:objective}
Considering the potentially varying contribution of $\mathbf{h}$ and $\mathbf{p}$ 
to the trustworthiness score,
we introduce two separate linear transformations to them, which are parameterized by
$\mathbf{W}_h\in\mathbb{R}^{C\times CK}$ and $\mathbf{W}_p \in \mathbb{R}^{C \times C}$, respectively. 
We then aggregate the two resultant vectors using an operator
$\aggoperator: \mathbb{R}^{C} \times \mathbb{R}^{C} \mapsto \mathbb{R}^{C}$,
which outputs a $C$-dimensional trustworthiness vector of $\mathbf{x}$ (one element for one class),
\begin{equation}
\label{eq:agg}
\mathbf{t}=\aggoperator\left(\mathbf{W}_h \mathbf{h}, \mathbf{W}_p \mathbf{p}\right).
\end{equation}
Here, the aggregation operator \aggoperator can be instantiated by any neural network.





The optimization process is carried out by reducing the negative log-likelihood loss
(NLL), i.e. $\mathbb{E}_{(\mathbf{x},\mathbf{y}) \sim (\mathcal{X}_{val}, \mathcal{Y}_{val})}[\mathcal{L}(\mathbf{x},\mathbf{y})]$ where each sample's loss is calculated as
\begin{equation}
\label{eq:loss}
     \mathcal{L}(\mathbf{x}, \mathbf{y}) =  -\frac{1}{ C} \sum_{c=1}^{C} \mathbf{y}_{c}\log(\mathbf{t}_{c}).
\end{equation}
The overall training procedure is summarized in Algorithm~\ref{alg:train}.

For simplicity, a single-layer feedforward neural network with a learnable weight matrix $\mathbf{W} \in \mathbb{R}^{C \times 2C}$ and a nonlinear activation $\sigma(\cdot)$ is used as the aggregator in this paper. 
Formally, the trustworthiness vector $\mathbf{t}$ (as shown in \figureautorefname{~\ref{fig:algo_image}}) can be expressed as  
\begin{equation}\label{eq:ti}
\mathbf{t}= \operatorname{softmax}  \left(\mathbf{W}^{T} \sigma\left(\left[\mathbf{W}_h \mathbf{h} \mathbin\Vert \mathbf{W}_p \mathbf{p}\right]\right)\right).
\end{equation}

Underlying the learnable framework, how do these two pieces of information cooperate during the aggregation? Curious about this question, we also investigate the mechanism and show the empirical result in \sectionautorefname{~\ref{sec:weightexp}}.

\paragraph{Inference.} 
Given a test sample $\Tilde{\mathbf{x}}$, we construct its corresponding neighborhood vector $\Tilde{\mathbf{h}}$ from the training dataset and fetch its classifier output vector $\Tilde{\mathbf{p}}$ from the base classifier $\mathcal{F}$. 
Then we evaluate the trustworthiness vector $\Tilde{\mathbf{t}}$ using \eqref{eq:agg} and fitted model parameters $\mathbf{W}, \mathbf{W}_h$ and $\mathbf{W}_p$. 
Finally, the trustworthiness score can be evaluated by indexing the trustworthiness vector using predicted class $c^*$, i.e., $\Tilde{t}_{c^*}$.
 

\subsection{Relation to Graph Neural Networks}
\label{relation-gnn}


In this section, we study the relations between our design and graph neural networks (GNNs) and show that our approach is inherently more flexible than GNNs in terms of aggregating neighborhood information 
to augment the classifier output for trustworthiness prediction. 
GNNs~\citep{kipf2016semi, xu2018powerful} have been a topic of interest in recent times for their powerful modeling capacity to aggregate neighbors, 
and this motivates us to compare our method with GNNs.
Among the several GNN variants, we choose the widely used graph convolutional neural network (GCN) 
as the subject for simplicity.

First, we show that our design of employing only one-hop neighbors for trustworthiness prediction is effective and efficient by comparing the performance of multi-hop GCNs with one-hop GCNs. 
Empirically, we demonstrate that the use of multi-hop GNNs does not have significant improvement and even degrades the performance for some datasets (see \tableautorefname{~\ref{tab:image-res}}). 
We argue that multi-hop neighborhood aggregation may lead to the over-smoothing issue and the noise accumulation risk, at least for our task.




Second, we prove that \method is essentially a \emph{generalized} version of a one-hop GCN: 
when imposing certain constraints on our \method (i.e., fixing the learned matrices $\mathbf{W}_h$ and $\mathbf{W}_p$ to be block diagonally-dominant), \method acts as a one-hop GCN. 
This equivalence is rigorously characterized as follows.
\begin{theorem}[One-hop GCN Equivalence]
\label{tho:equi}
Provided that $\mathbf{W}_h$ exhibits a block diagonal structure:
\begin{equation*}
    \mathbf{W}_h = {\frac{1}{K}}\Bigl[ I_{C \times C} \otimes \boldsymbol{1}^T \Bigr]  \, \mathrm{with} \, 
    \boldsymbol{1}^T = \underbrace{[1, 1, \cdots, 1]}_{K\, 1\text{'s}},
\end{equation*}
where $\otimes$ denotes the Kronecker product, and that $\mathbf{W}_p = I_{C \times C}$, 
\method operates as a one-hop Graph Convolutional Network with the node features 
$[\mathbf{y} \mathbin\Vert \mathbf{0}] \in \Delta^{2C-1}$ for $y \in \mathcal{Y}_{tr}$ and 
$[\mathbf{0} \mathbin\Vert \mathbf{p}] \in \Delta^{2C-1}$ for $\mathbf{p} = \mathcal{F}(\mathbf{x})$ with
$\mathbf{x} \in \mathcal{X}_{val} \cup \mathcal{X}_{ts}$, 
and the adjacency matrix $\mathbf{A}$ induced by a predefined kernel
$\mathcal{K}_{f}$ (e.g. Laplacian kernel). 
\end{theorem}

\begin{proof}
The proof is relegated to \appendixautorefname{~\ref{app:gcn-proof}}.
\end{proof}



\paragraph{Remark.}
In fact, our approach is more flexible than one-hop GCN for feature aggregation, 
as $\mathbf{W_h}$ and $\mathbf{W_ p}$ in our setting can exhibit more flexible forms than the simple block diagonal structure. 
This is further verified by empirical studies (see \figurename{~\ref{fig:W1W2}}), which show that 
our model can exploit not only intra-class relations, 
but also inter-class relations, which one-hop GCN cannot.
More detailed analyses can be found in Section~\ref{sec:weightexp}.

\begin{algorithm}[!t]
    \caption{\small \methodcmd}
    \label{alg:metcmd}
    \SetAlgoLined
    \KwIn{Dataset $\{ (\mathbf{x_i},y_i, \hat{y}_i) \}_{i=1}^{N}\}$; Mislabeling rate $p$; Confidence level $\alpha$;
          Well-trained trustworthiness model \method.}
    \KwOut{Mislabeled sample set $\mathcal{S}$.} %
    
    
    $\mathcal{T} = \{\mathbf{t}_i ~ \vert ~ \mathbf{t}_i =  \method(\mathbf{x}_i), \quad \forall ~ 1 \le i \le n \}$
    
    $\mathcal{R} = \{r_i ~ \vert ~ r_i = (2 \cdot \mathbb{I}(\hat{y}_i = y_i) - 1) \cdot t_{iy_i}, \quad \forall ~ \mathbf{t}_i \in \mathcal{T}\}$

    $\mathcal{R} = \operatorname{sort}(\mathcal{R})$  \hfill \linecomment{Sort in non-increasing order}
    
    $B_\alpha = \left\lceil(N+1)(1- \alpha)  +\alpha N p ) \right\rceil$ 
    
    $\tau_\alpha = r_{(B_\alpha)}$  \hfill \linecomment{$r_{(B_\alpha)}$ is the $B_\alpha$-th largest element of $\mathcal{R}$}
    
    
    $\mathcal{S} = \{ (\mathbf{x}_i, y_i) ~ \vert ~ r_i \le \tau_{\alpha}, \quad \forall ~ 1 \le i \le N \}$
   
    \Return $\mathcal{S}$
\end{algorithm}

\subsection{Conformal Mislabel Detection: An Extension of \method}

In this section, we show that our \trustscore can also be used for another task, mislabel detection. 
In particular, we introduce the \methodcmd algorithm to assess the reliability of data labels in a mislabeled dataset. 
The detailed procedure is described in \algorithmautorefname{~\ref{alg:metcmd}}.

To identify mislabeled data in a noisy-labeled dataset, we compute a reliability score for the label of each sample using \method and a well-trained base classifier.
The reasoning behind this is that 
\emph{labels that contradict the classifier's prediction and have low trustworthiness scores are questionable}, meaning that when the neighborhood supports the classifier's prediction rather than the label itself, it is more likely the label that makes a mistake.
So we devise the \emph{reliability score} based on the class label $y$'s \trustscore $t_{y}$:
\begin{equation}
    r = (2 \cdot \mathbb{I}(\hat{y} = y) - 1) \cdot t_{y},
\end{equation}
where the indicator function $\mathbb{I}(\cdot)$ and the classifier's prediction $\hat{y}$ are used to detect whether the sample is misclassified. 
Samples with reliability scores lower than a certain threshold $\tau_{\alpha}$ are treated as mislabeled (i.e. $r < \tau_{\alpha}$).

To bound the probability of false negative detections, we apply the conformal anomaly detection~\citep{balasubramanian2014conformal} framework and extend the existing work to the noisy setting for determining the threshold $\tau_{\alpha}$.  
Given a dataset of size $N$ with 
mislabeling rate $p$, and a user-specified confidence level $\alpha$, 
we compute reliability scores for each sample in the validation dataset and sort them in non-increasing order as $(r_{(1)}, \dots, r_{(N)})$.
The threshold $\tau_\alpha$ is set to the $B_\alpha$-th largest element, i.e.,
\begin{equation}
    \label{eq:metcmd}
    \tau_\alpha = r_{(B_\alpha)}, \text{ where } B_\alpha = \left\lceil(N+1)(1- \alpha)  +\alpha N p \right\rceil.
\end{equation}



Next, we show the following theoretical guarantee which to the best of
our knowledge is the first to consider the real-world noisy data setting. 

\begin{theorem}[Noisy-robust Coverage Guarantee]
    \label{theo:convergency}
    For any given confidence level $\alpha \in (\frac{1}{N+1}, 1)$,
    with probability at least $1-\alpha$ over the random choice of any correctly labeled data point $(\Tilde{\mathbf{x}}, \Tilde{y})$, we have 
    $$ \Tilde{r} > \tau_\alpha, $$
    where $\Tilde{r}$ is the predicted reliability score of $\Tilde{\mathbf{x}}$ and 
    $\tau_\alpha$ is defined in \equationautorefname{~(\ref{eq:metcmd})}.
\end{theorem}
\begin{proof}
The detailed proof is relegated to \appendixautorefname{~\ref{app:mislabel-proof}}. In contrast to existing work in conformal learning which requires an i.i.d. validation set, our algorithm uses a partitioning approach to allow for a more realistic setting involving a small percentage of mislabeled samples.
\end{proof}

\paragraph{Remark.}
Theorem~\ref{theo:convergency} suggests that the reliability score that our method outputs provide a theoretical guarantee --- with high probability, a correctly labeled data point will be given a score above the threshold $\tau_\alpha$. In other words, if we select the likely mislabeled samples by selecting those below the reliability threshold, we can bound the probability of a false positive (by $\alpha$).


\section{Experiments}
\label{sec:exp}

Through extensive experiments, we aim to answer the following questions: 
\begin{compactitem}
    \item{\bf{Mechanism Visualization}:} How does \method work?
    \item{\bf{Effectiveness}:} How well does \method perform on different types of datasets?
    \item{\bf{Ablation Study}:} How does each component of \method contribute to the trustworthiness performance?  
    \item{\textbf{Sensitivity}:} How sensitive to hyperparameters is \method?
    \item{\textbf{Computational Cost}:} How fast is \method?
    \item{\bf{Case Study}:} How is \method extended to mislabel detection?
\end{compactitem}
Due to space limitation, we refer discussions about hyperparameter sensitivity, computational cost, and case study of mislabel detection to the \appendixautorefname{~\ref{app:hyperparameter-sensitive}, ~\ref{app:comp-cost}, ~\ref{app:mislabel}}.






\subsection{Mechanism Visualization and Verification}
\label{sec:weightexp}

\begin{figure}[!t]
    \centering    
    \includegraphics[width=0.8\linewidth]{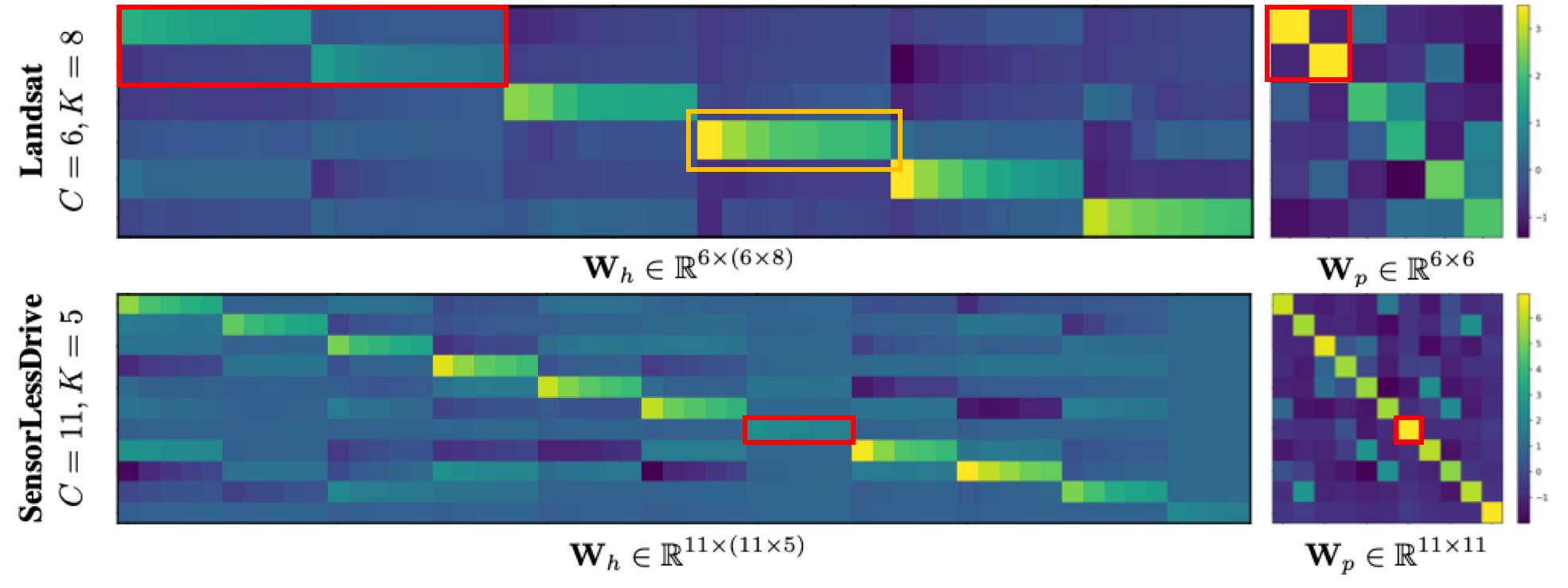}
    \caption{\textbf{Visualization of weight matrices $\mathbf{W}_h$ for similarity vectors and $\mathbf{W}_p$ for classifier output learned by \method}. Brighter colors indicate larger values. The dimmer diagonal blocks in $\mathbf{W}_h$ (e.g. in the red rectangle) are empirically associated with their corresponding brighter diagonal entries in $\mathbf{W}_p$, suggesting that \method combines sample confidence with neighborhood information in a \emph{complementary} manner. The different weights within each diagonal block in $\mathbf{W}_h$ (e.g. in the orange rectangle) suggests that \method can learn an appropriate $K$ (i.e. how many neighbors are necessary to determine the trustworthiness) for every class based on its local density, making it hyperparameter-insensitive.
    }
    \label{fig:W1W2-onerow} 
\end{figure}

In this section, we examine the mechanism of \method empirically, demonstrating that our method utilizes neighborhood and classifier information in a complementary manner and captures intra-class and inter-class relations inside the neighborhood.

\paragraph{Complementary effect}
We show that \method integrates neighborhood information with classifier output in a \emph{complementary} manner: it adaptively weighs the importance of the classifier's output against neighborhood information for each class. In other words, when the sample's neighborhood information is not accurate or useful, it relies more on the classifier output; and vice versa.
To show this empirically, \figurename{~\ref{fig:W1W2-onerow}} suggests that the dimmer diagonal blocks in $\mathbf{W}_h$ (e.g. in the red rectangle) are empirically associated with their corresponding brighter diagonal entries in $\mathbf{W}_p$, and vice versa. To further confirm this quantitatively, we calculate the Pearson's correlation coefficient $\rho$
between the diagonal blocks in $\mathbf{W}_h$ and the corresponding diagonal entries in $\mathbf{W}_p$ on Landsat.
The result shows a strong negative correlation $\rho=-0.90$, which sheds light on the complementary mechanism in \method. 



\paragraph{Intra-class and inter-class relations} We have three observations on how \method leverages neighbors: firstly, aligning with our motivation and \theoremautorefname{~\ref{tho:equi}}, the learned weight matrices $\mathbf{W}_h$ and $\mathbf{W}_p$ are significantly block-diagonally dominant as shown in \figurename{~\ref{fig:W1W2-onerow}}, 
exhibiting the neighbor-homophily property and our method's similarity to graph neural networks.
Secondly, the different weights within each diagonal block in $\mathbf{W}_h$ reflect different importance among neighbors of the same class (i.e. \emph{intra-class} proximity) as shown in the orange rectangle region of \figurename{~\ref{fig:W1W2-onerow}}. This suggests that \method can automatically learn how many neighbors are necessary to determine the trustworthiness for every class based on its neighborhood without tuning the hyperparameter $K$.
Thirdly, the off-diagonal blocks represent the inter-class relations, such as similar or exclusive relations among classes, which makes it more flexible in 
determining a robust \trustscore. We leave the more detailed discussion in \appendixautorefname{~\ref{app:mechanism}}.



\subsection{Experiment Setup}

\paragraph{Datasets.}

We evaluate our method on image datasets including CIFAR10 \citep{Krizhevsky09learningmultiple}, FashionMNIST \citep{xiao2017/online} and MNIST \citep{deng2012mnist}, and UCI tabular datasets \citep{Dua:2019}, including CardDefault, Landsat and LetterRecognition, etc. Statistics of each dataset are summarized in \appendixautorefname{~\ref{app:exp-setup}}. 

\paragraph{Compared methods.}
We compare our proposed \method with the following methods:
\begin{compactitem}
    
    \item \textbf{Confidence Score} \citep{hendrycks17baseline} employs the maximum softmax output of a classifier as a measure of trustworthiness. 
    \item \textbf{Temperature Scaling} \citep{guo2017calibration} modifies the confidence score using a temperature parameter $T$ learned from the validation set. 
    \item \textbf{TCP} \citep{corbiere2019addressing} trains a ConfidNet using an intermediate output of any neural networks as the input for regression to the desired softmax output. 
    
    \item \textbf{Trust Score} \citep{jiang2018trust} defines the trustworthiness measure as the ratio between the distance from the test sample to its nearest neighbor with labels excluding the predicted class, and the distance from the test sample to the nearest neighbor of the predicted class.
    
    \item \textbf{Top-label Calibration} \citep{gupta2021distribution} calibrates a classifier's softmax output by using histogram binning to reduce top-label multi-class calibration into binary calibration.
    \item \textbf{GCN-khop} uses k-hop GCN to aggregate neighborhood information and classifier output for trustworthiness prediction. Its detailed implementation can be found in \theoremautorefname{~\ref{tho:equi}} and \appendixautorefname{~\ref{app:gnn-algo}}. 
\end{compactitem}

\paragraph{Evaluation Metrics.}
\label{sec:evaluation-metrics}
Following the existing pioneering work on trustworthiness~\citep{hendrycks17baseline, corbiere2019addressing}, we adopted the \emph{same} metrics to evaluate the trustworthiness of a base classifier: AUC-ROC, APM and APC.
More details regarding specific evaluation procedures can be found in Sec. 2 of \citet{hendrycks17baseline}. All reported results are averaged over $5$ trials under distinct random seeds on the same splits of datasets. 


\paragraph{Implementation Details.}
 For tabular datasets, experiments are conducted based on three base classifiers, including logistic regression (LR)~\citep{peng2002introduction}, random forest (RF)~\citep{svetnik2003random} and multi-layer perceptrons (MLPs)~\citep{ruck1990multilayer}; while for image datasets, shallow convolutional networks, Resnet18 and Resnet50 \citep{he2016deep} are used.  We leave details such as hyperparameters to \appendixautorefname{~\ref{app:exp-setup}}.

\subsection{Effectiveness of \method} 

\begin{table*}[t]
    \small
    \centering
    \resizebox{\linewidth}{!}{
    \begin{tabular}{ll|ccc|ccc|ccc}  \toprule
    \multirow{2}{*}{Clf} & \multirow{2}{*}{Method} & \multicolumn{3}{c|}{LetterRecognition}                              & \multicolumn{3}{c|}{Landsat}                                        & \multicolumn{3}{c}{CardDefault}                                    \\
                     &                      & AUC \%               & APC \%               & APM \%               & AUC \%               & APC \%               & APM \%               & AUC \%               & APC \%               & APM \%               \\ \midrule
\multirow{6}{*}{LR}  & Confidence                                        & 85.28(0.18)          & 95.03(0.14)          & 61.25(0.51)          & 88.05(0.47)          & 97.77(0.07)          & 52.09(2.13)          & 65.19(0.30)          & 86.94(0.21)          & 33.44(0.32)          \\
                     & TempScaling                                        & 84.67(0.22)          & 94.83(0.15)          & 59.72(0.64)          & 87.20(0.42)          & 97.63(0.08)          & 48.85(1.73)          & 65.22(0.33)          & 87.04(0.25)          & 33.50(0.36)          \\
                     & TrustScore                                   & 95.75(0.23)          & 98.46(0.12)          & 86.93(0.63)          & 91.55(0.39)          & 98.39(0.12)          & 64.76(0.70)          & 61.61(0.42)          & 85.35(0.28)          & 28.42(0.41)          \\
                     & TCP                                         & 90.78(0.21)          & 96.96(0.13)          & 74.85(0.43)          & 89.47(0.49)          & 98.06(0.15)          & 54.25(1.78)          & \textbf{68.79(0.29)}         & \textbf{88.78(0.22) }         & 34.14(0.37)          \\
                     & TopLabel                                      & 78.58(0.31)          & 92.63(0.18)          & 44.85(0.53)          & 84.05(0.31)          & 96.45(0.10)          & 41.55(1.10)          & 64.80(0.45)          & 86.68(0.30)          & 33.91(0.40)          \\
                     & Ours                                        & \textbf{99.08(0.04)} & \textbf{99.72(0.01)} & \textbf{97.17(0.13)} & \textbf{93.40(0.17)} & \textbf{98.84(0.04)} & \textbf{72.54(1.40)} & 67.60(0.31)          & 87.45(0.22)          & \textbf{36.06(0.33)}  \\ \midrule
\multirow{6}{*}{RF}  & Confidence                                        & 93.94(0.29)          & 99.48(0.03)          & 51.41(1.97)          & 90.25(0.37)          & 98.69(0.09)          & 48.77(1.78)          & 68.89(0.32)          & 89.43(0.11)          & 33.24(0.37)          \\ 
                     & TempScaling                                        & 94.58(0.19)          & 99.55(0.02)          & 55.41(1.86)          & 89.26(0.24)          & 98.56(0.05)          & 46.88(0.91)          & 68.68(0.32)          & 89.31(0.16)          & 33.07(0.44)          \\
                     & TrustScore                                   & 90.96(0.23)          & 99.22(0.01)          & 39.19(1.46)          & 88.52(0.34)          & 98.51(0.04)          & 43.36(2.76)          & 59.68(0.29)          & 84.89(0.23)          & 25.84(0.39)          \\
                     & TCP                                         & 85.83(0.22)          & 98.71(0.06)          & 29.40(0.53)          & 85.07(0.77)          & 97.97(0.16)          & 34.30(1.62)          & 67.96(0.12)          & 89.57(0.14)          & 30.08(0.24)          \\
                     & TopLabel                                      & 83.99(0.51)          & 98.44(0.10)          & 26.79(0.55)          & 84.44(0.30)          & 97.57(0.11)          & 32.54(1.10)          & 67.64(0.36)          & 88.90(0.20)          & 32.14(0.38)          \\
                     & Ours                                        & \textbf{96.45(0.18)} & \textbf{99.69(0.02)} & \textbf{72.16(1.36)} & \textbf{91.23(0.26)} & \textbf{98.91(0.06)} & \textbf{53.60(1.80)} & \textbf{69.27(0.30)} & \textbf{89.61(0.09)} & \textbf{34.27(0.45)} \\ \midrule
                     
                                                                     & Confidence                                         & 90.71(0.18)          & 99.18(0.04)          & 39.59(0.93)          & 84.41(1.61)          & 96.95(0.71)          & 40.26(2.04)          & 68.99(0.32)          & 89.05(0.18)          & 34.17(0.49)          \\
                                                                     & TempScaling                                         & 93.83(0.15)          & 99.49(0.01)          & 52.68(0.58)          & 87.10(0.53)          & 97.98(0.16)          & 46.16(1.20)          & 68.46(0.39)          & 88.96(0.20)          & 34.69(0.45)          \\
                                                                     & TrustScore                                    & 88.53(0.31)          & 99.05(0.05)          & 32.28(0.60)          & 88.09(0.48)          & 98.32(0.09)          & 41.85(1.47)          & 60.20(0.39)          & 84.82(0.25)          & 26.60(0.35)          \\
                                                                     & TCP                                          & 79.91(0.58) & 97.62(0.15) & 25.61(1.06) & 86.01(1.01) & 97.78(0.26) & 39.70(2.67) & 67.77(0.19) & 88.91(0.09) & 31.17(0.27)        \\
                                                                     & TopLabel                                       & 78.78(1.55) & 98.09(0.20) & 16.65(0.79) & 81.29(1.02) & 96.70(0.29) & 30.16(0.93) & 67.64(0.37) & 88.68(0.15) & 32.81(0.44) \\
\multirow{-6}{*}{MLP}                                                & Ours                                         & \textbf{95.02(0.36)} & \textbf{99.58(0.04)} & \textbf{65.81(1.19)} & \textbf{91.75(0.39)} & \textbf{98.88(0.08)} & \textbf{57.80(0.91)} & \textbf{69.69(0.29)} & \textbf{89.51(0.19)} & \textbf{35.64(0.23)} \\
                                           \bottomrule
    \end{tabular}
    }
    \caption{The performance of our proposed model \method and other models on three tabular datasets (\textbf{mean$\pm$std}). We report the results of all models based on different base classifiers (LR, RF, MLP)
    and best results are emphasized in bold. We use \emph{TempScaling} for Temperature Scaling due to space limitation.
    }
    \label{tab:tabular-res}
\end{table*}

Performance results on tabular datasets and image datasets are summarized in \tablename{~\ref{tab:tabular-res}} and \tablename{~\ref{tab:image-res}} respectively, from which we make the following observations:

Firstly, our proposed \method outperforms other models under almost all metrics across benchmarks. Specifically, \tablename{~\ref{tab:tabular-res}} shows that our model achieves the most significant improvement on APM, with the highest performance gain of 12.17\%, and the average performance gain of 7.63\%. 
This suggests that our model performs best in identifying misclassified samples. Besides, \method achieves more than 2\% improvement on AUC in most cases. \tablename{~\ref{tab:image-res}} reveals that our model also achieves better or comparable performance on image datasets.

Secondly, the result suggests that the neighborhood information and classifier prediction are two essential and complementary sources of information for trustworthiness prediction. 
This is demonstrated by results shown in \tablename{~\ref{tab:tabular-res}} that Trust Score achieves better results than Confidence Score on LetterRecognition and Landsat when LR is the base classifier, whereas Confidence Score performs better on CardDefault.
Moreover, our method's outperformance of both information sources validates the complementary effect. 

Thirdly, our method consistently beats the GCN-based method across all datasets, suggesting that our formulation is more effective and efficient. In addition, \tableautorefname{~\ref{tab:image-res}} also demonstrates that one-hop neighborhood aggregation is sufficient and that utilizing multi-hop neighbors may lead to the over-smoothing issue, by showing that multi-hop graph convolutional neural networks have limited improvement compared to the one-hop model, and sometimes become worse.



\begin{table*}[t]
\small \centering
\resizebox{\linewidth}{!}{
\begin{tabular}{l|lll|lll|lll} \toprule
\multirow{2}{*}{Method} & \multicolumn{3}{c|}{MNIST}                                                            & \multicolumn{3}{c|}{FashionMNIST}                                                     & \multicolumn{3}{c}{CIFAR10}                                                          \\
                        & \multicolumn{1}{c}{AUC \%} & \multicolumn{1}{c}{APC \%} & \multicolumn{1}{c|}{APM \%} & \multicolumn{1}{c}{AUC \%} & \multicolumn{1}{c}{APC \%} & \multicolumn{1}{c|}{APM \%} & \multicolumn{1}{c}{AUC \%} & \multicolumn{1}{c}{APC \%} & \multicolumn{1}{c}{APM \%} \\ \hline
Confidence                    & 90.48(0.39)                & 98.93(0.06)                & 46.71(1.93)                 & 91.31(0.32)                & 99.10(0.03)                & 44.22(1.03)                 & 83.72(1.21)                & 94.97(0.66)                & 54.42(1.66)                \\
TempScaling                   & 90.50(0.40)                & 98.93(0.06)                & 47.27(1.99)                 & 91.33(0.31)                & \textbf{99.11(0.03)}       & 44.27(0.99)                 & 83.75(1.29)                & 94.96(0.68)                & 54.81(1.80)                \\
TrustScore                  & \textbf{96.40(0.28)}                & \textbf{99.61(0.04)}                & 78.53(1.25)                 & 91.31(0.21)                & 99.10(0.03)                & 47.27(0.86)                 & 86.98(0.75)                & 96.04(0.18)                & 63.29(3.84)                \\
TCP                     &   92.11(1.03)	             &	98.37(0.43)	                &	69.91(7.47)             &                           90.82(0.07)	&	98.83(0.03)	&	\textbf{50.52(1.33)}                            &  86.63(0.92)	&	95.36(0.16)	&	64.06(3.40)                \\
TopLabel                   & 90.35(0.31)&98.89(0.05)&43.31(1.02)&89.54(0.41)&98.62(0.17)&45.76(1.53)               & 85.24(1.17)&94.40(0.34)&60.90(5.59)                \\
Mahala                   & 75.88(0.23)& 96.91(0.03)& 21.96(0.48)&59.87(0.95)&94.26(0.06)& 11.15(0.08)               & 55.73(5.92)&82.93(3.07)&23.47(3.10)                \\
\midrule
GCN3hop                    & 91.77(0.27)	& 99.07(0.05)	&	55.75(1.77)       & 90.44(0.54)	&	98.94(0.08)	&	45.63(1.70)	        & 82.68(1.44)	& 93.78(0.31)	& 60.79(5.30)   \\ 
GCN2hop                    & 91.58(0.29)	& 99.05(0.05)	&	54.71(1.90)        & 90.35(0.56)	&	98.92(0.07)	&	45.63(1.91)        & 83.75(1.66)	& 94.18(0.41)	& 62.37(5.46)    \\ 
GCN1hop                    & 91.38(0.30)	& 99.02(0.06)	&	53.77(1.89)       & 90.31(0.56)	&	98.92(0.08)	&	45.60(1.76)	        & 84.08(1.46)	& 94.33(0.27)	& 62.63(5.39)   \\ 
Ours                    & \textbf{96.40(0.52)}       & 99.55(0.08)                & \textbf{81.02(1.91)}        & \textbf{91.52(0.16)}       & 99.04(0.02)                & 48.83(0.51)        & \textbf{87.52(1.07)}       & \textbf{96.05(0.51)}       & \textbf{65.27(4.03)}    \\ \bottomrule

\end{tabular} }
\caption{The performance of \method and other models on image datasets (\textbf{mean $\pm$ std}). 
Best results are emphasized in bold. We refer to \tableautorefname{~\ref{tab:tabular-res}} for the full name of abbreviations. 
}
\label{tab:image-res}
\end{table*}

\subsection{Ablation Study}

\begin{figure*}[t]
    \small
    \centering
    \subfigure[Comparison for LR base classifiers under the APM metric. \label{fig:abl_1}]
    {\includegraphics[width=0.7\linewidth]{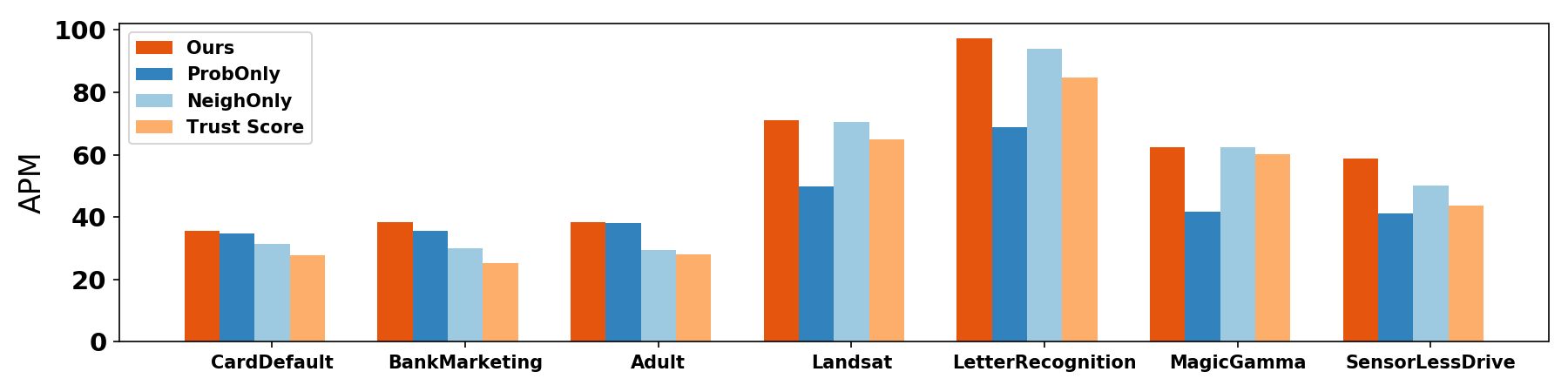} }
    \vspace{-0.1in}
    \subfigure[Comparison for RF base classifiers under the AUC metric.  \label{fig:abl_2}]
    {\includegraphics[width=0.8\linewidth]{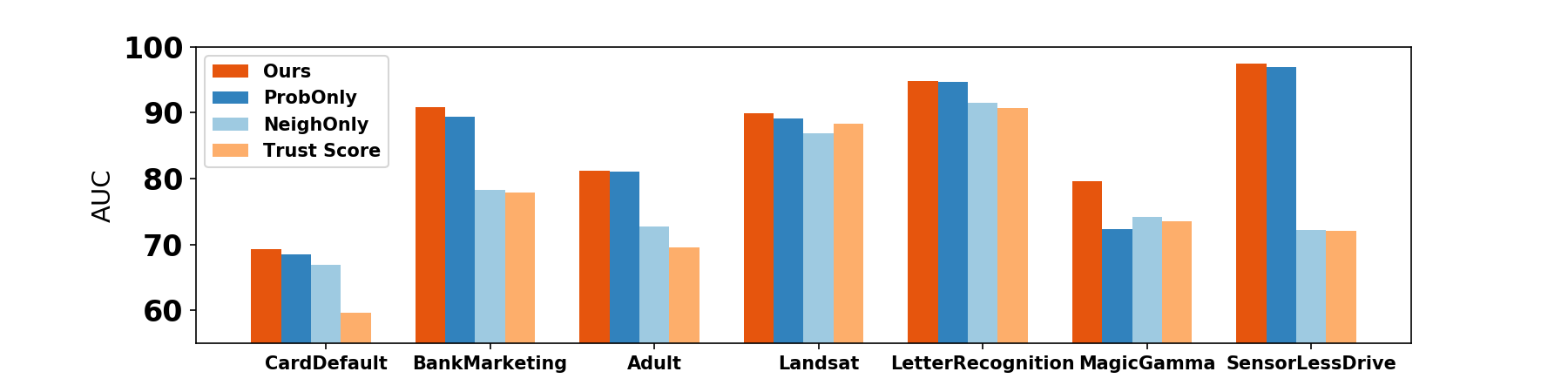} }
    \caption{\textbf{Performance results for ablation study.} 
       \method outperforms all the other model variants \predonly and \neionly across all datasets
        with LR and RF as the base classifiers, respectively. More results can be found in \appendixautorefname{~\ref{app:ablation-app}}.
      }
    \label{fig:ablation}
\end{figure*}

To demonstrate the effectiveness of each component and 
the adaptiveness of the \emph{learnable} weights in \method, 
we compare against its variants \predonly and \neionly,
\begin{compactitem}
    \item{\neionly:} solely takes neighborhood vectors as input,
    \item{\predonly:} solely inputs classifier output vectors,
\end{compactitem}
as well as the non-learnable baseline Trust Score.

\figurename{~\ref{fig:ablation}} demonstrates the comparison results of the ablation study on seven tabular datasets. Results using other base classifiers and other metrics are listed in \appendixautorefname{~\ref{app:ablation-app}}. 

We note that our \method 
consistently outperforms \predonly and \neionly across all datasets, especially on SensorLessDrive with 8.83\% gain and on BankMarketing with 3\% gain, 
which indicates that both vectors make non-negligible contributions to the final \trustscore, 
and that considering either of them alone is insufficient.
It supports the claim that the neighborhood and the predictive output of the classifier complement one another in determining the \trustscore. 

Moreover, the comparison between \neionly and Trust Score suggests that 
considering a set of neighborhood rather than solely class-wise nearest neighbors contributes to the performance gain;
by inspecting those neighborhoods, \method can utilize richer information and 
flexibly choose its receptive field, i.e., 
how many neighbors are necessary to determine the \trustscore. 
It also empowers the model to adaptively capture intra-class relations (i.e. different proximity of a sample's neighbors)
and inter-class relations, e.g., some classes may be more closely related as compared to other classes.

\subsection{Mislabel Detection: A Case Study}
\label{subsec:mislabel}

This section demonstrates the usefulness of our \methodcmd algorithm by presenting mislabeled samples in real-world datasets.
We use a dataset named \textit{QuoraInsQ} from the Kaggle competition ``Quora Insincere Questions Classification'' that aims to improve online environment by detecting toxic questions. The dataset consists of 1,306,122 questions which are manually categorized as sincere or insincere. The definition of an insincere question is one that intends to make a statement instead of eliciting helpful responses. In order to estimate the mislabeling rate, we manually relabel 500 randomly selected questions and utilize the fraction of incorrectly labeled samples as the mislabeling rate. 

Firstly, we use the model of the top-ranking team from the leaderboard as our base classifier and use \methodcmd to compute the reliability score for each sample. The mislabeling rate $p$ is estimated as $0.03$. Then, we run \methodcmd with the confidence level $\alpha=5\%$ and obtain the detected mislabeled results. 
Then we showcase some of the detected example questions with the lowest reliability scores in \tableautorefname{~\ref{tab:mislabel}}. We find that all of them were labeled as \emph{`insincere'} in the original dataset, but none of them breach the four rules that signify a question as insincere.
More experiments can be found in \appendixautorefname{~\ref{app:mislabel}}. 

\begin{table}[t]
    \centering
    \resizebox{\linewidth}{!}{
    \begin{tabular}{l}  \toprule
    \textbf{Detected mislabeled example questions}         \\ \hline
    What is the past tense of past tense?                                                                                                                                                                                      \\ \hline
    What is the difference between a fusion and a restaurant?                                                                                                                                         \\ \hline
    What are the new product for agent project?                                                                                                                                                    \\ \hline
    \begin{tabular}[c]{@{}l@{}}Which protagonist from a video game have you most related to?\end{tabular}                                                                                                      \\ \hline
    \begin{tabular}[c]{@{}l@{}}Do I have to appear for IMU CET again even if I get a good rank in it if I'm appearing for  improvement of  HSC board exam?\end{tabular}                                                 \\ \hline
    \begin{tabular}[c]{@{}l@{}}What are some important things/steps when starting a 
    film production company in Netherlands?\end{tabular}                                                                             \\ \hline
    \begin{tabular}[c]{@{}l@{}}What astrological combinations are needed to obtain a scholarship for studies?\end{tabular}                                                                       \\ \hline
    \begin{tabular}[c]{@{}l@{}}What advice would you give a person intending to buy a Nissan note, in terms of performance \\ i.e. traction,  fuel economy, maintenance and resale?\end{tabular}    \\ \bottomrule
    \end{tabular}
    }
    \caption{\small Mislabeled samples detected in \textit{QuoraInsQ} by our \methodcmd. These questions are labeled as insincere but are actually sincere.}   
    \label{tab:mislabel}
\end{table}

\section{Related Work}
\label{sec:background}




\paragraph{Trustworthiness Prediction.}
Trustworthiness prediction, also known as ``failure prediction'' and ``misclassification detection'' in the literature, aims to assign a discriminative score to every prediction given by a base classifier, indicating whether we can trust this prediction or not.
This has received increasing attention in recent times.
\citet{hendrycks17baseline} suggests using confidence score to tackle this problem. Therefore, the confidence calibration method~\citep{Kull2019BeyondTS, gupta2021distribution}, designed to mitigate the over-confidence issue of the confidence score, can also be applied to trustworthiness prediction by giving a more accurate calibrated confidence score.
Monte-Carlo dropout~\citep{Gal2016DropoutAA} and Deep-Ensemble~\citep{Lakshminarayanan2017SimpleAS} compute the output variance of multiple trials to detect incorrect predictions, while these ensemble-based methods are quite computationally expensive. Trust Score \citep{jiang2018trust} proposes a score which is a fixed, non-learnable function of the neighborhood of a sample, and hence suffers from limited functional space.
\citet{corbiere2019addressing} proposes a regression method to fit the ground truth label's corresponding softmax score and uses it as a proxy for trustworthiness. However, this relies on the assumption that the base classifier is always over-confident, which is not always the case (e.g. graph neural networks were found to be under-confident in \citet{Wang2021BeCT}).
\citet{Malinin2018PredictiveUE, Malinin2020EnsembleDD, Sensoy2018EvidentialDL, Charpentier2020PosteriorNU} assume the classifier outputs are sampled from a latent Dirichlet distribution and treat low-likelihood samples as misclassified samples. These methods typically involve modified architectures that need to be 
trained from scratch, and in some cases can involve the trade-off between classification accuracy and the performance of trustworthiness prediction. 
In contrast, our proposed \method keeps the base classifier intact and uses auxiliary information for simple estimation. 
Similarly, \cite{deng2022trust} extracts auxiliary information from self-supervised pretext tasks for the trustworthiness prediction.
In our work, we aim to measure trustworthiness by adaptively utilizing the classifier's predictive output and neighborhood information via a flexible mapping that combines the best of both worlds. 
\paragraph{Relations to Out-of-distribution Detection.}
Out-of-distribution (OoD) detection \citep{hendrycks17baseline, sastry2020detecting, ming2022posterior} are closely related to trustworthiness prediction but targeted at a different goal. OoD detection aims to measure the sample quality by identifying input data whose ground truth label is not covered by the label set of the training dataset, whereas trustworthiness prediction aims to measure the classifier's quality by identifying input data whose predicted label does not match its ground truth label. However, since they both detect abnormal behaviors, the assumption used for one task can be applied to the other with some modification. For example, our neighbor-homophily assumption can also be extended to OoD detection by assuming that samples that are distant from neighbors in all classes are likely to be out-of-distribution data. On the other hand, methods for detecting OoD data, such as mahalanobis distance \cite{lee2018simple}, can also be adapted to our task.

\paragraph{Mislabel Detection} The goal of mislabel detection is to identify data whose labelled class differs from the underlying ground truth class. A majority of research in this field leverages training dynamics for differentiating correctly labelled and mislabelled samples, such as the dynamics of logit in AUM~\cite{pleiss2020identifying} and the loss distribution in DY-Bootstrape\cite{arazo2019unsupervised}. Our approach falls into a different line of direction~\cite{https://doi.org/10.48550/arxiv.1911.00068,zhang2021learning} that employs pre-trained classifiers. Confident Learning~\cite{https://doi.org/10.48550/arxiv.1911.00068}, for example, estimates a noise transition matrix based on the softmax output. In addition to the softmax output, we also consider a sample's local neighborhood, which allows us to infer individual samples more accurately compared to confident learning.

\section{Conclusions and Discussion}
\label{sec:concl}

\paragraph{Conclusion}
Knowing when to trust a classifier is essential for safe deployment of present machine learning algorithms. To solve the problem, we devise a model-agnostic post-hoc trustworthiness prediction algorithm \method which leverages information from the neighborhood and the classifier to predict the trustworthiness of predictions given by a classifier. By theoretically demonstrating that \method is a generalized one-hop GCN aggregating information from the neighborhood and the sample itself, we provide a better understanding of how our approach works. The working mechanism is also revealed by empirical studies, which show that our approach can utilize the classifier output and neighborhood information in a complementary manner, and capture diverse similarities within different neighborhoods.
On several tabular and image benchmarks, the effectiveness of \method is empirically validated via 
comparison with other state-of-the-art methods and ablation studies. 
Of independent interest, an extension of \method to mislabel detection 
is also introduced with a noise-robust coverage guarantee for bounding the false negative predictions. 

\paragraph{Discussion}
The current study assesses a classifier's trustworthiness by utilizing each sample's neighborhood, comparing each sample's prediction to its neighbors' ground truth labels. Besides that, it would be interesting to explore alternative information associated with the trustworthiness of the classifier. The model explanation generated by explainable approaches, for example, can be used to determine trustworthiness. A prediction with implausible explaining logic is likely wrong. 
Moreover, going beyond unstructured data, we believe \method is also promising for measuring trustworthiness of graph-structured data, which can be an interesting and nontrivial extension to this work.

\subsubsection*{Acknowledgments}
This work was supported in part by NUS ODPRT Grant R252-000-A81-133.


\newpage
\appendix

\section{Proof of Theorem 1}
\label{app:gcn-proof}
Before we show the proof, we demonstrate how the column concatenation operator $\cdot\mathbin\Vert\cdot$ operates through an example: the column concatenation of two column vectors $(1, 2)^T$ and $(3, 4)^T$ yields
\begin{equation}
[{(1, 2)}^T \mathbin\Vert {(3, 4)}^T] = {(1, 2, 3, 4)}^T
\end{equation}

\begin{theorem}[One-hop GCN Equivalence]
Provided that $\mathbf{W}_h$ exhibits a block diagonal structure:
\begin{equation*}
    \mathbf{W}_h = {\frac{1}{K}}\Bigl[ I_{C \times C} \otimes \boldsymbol{1}^T \Bigr]  \, \mathrm{with} \, 
    \boldsymbol{1}^T = \underbrace{[1, 1, \cdots, 1]}_{K\, 1\text{'s}},
\end{equation*}
where $\otimes$ denotes the Kronecker product, and that $\mathbf{W}_p = I_{C \times C}$, 
\method operates as a one-hop Graph Convolutional Network with the lifted node features 
$[\mathbf{y}^{(i)} \mathbin\Vert \mathbf{0}] \in \mathbb{R}^{2C}$ for $y^{(i)} \in \mathcal{Y}_{tr}$ and 
$[\mathbf{0} \mathbin\Vert \mathbf{p}^{(i)}] \in \mathbb{R}^{2C}$ for $\mathbf{p}^{(i)} = \mathcal{F}(\mathbf{x}^{(i)})$ with
$\mathbf{x}^{(i)} \in \mathcal{X}_{val} \cup \mathcal{X}_{ts}$, 
and the adjacency matrix $\mathbf{A}$ induced by a predefined kernel
$\mathcal{K}_{f}$ (e.g. Laplacian kernel). 
\end{theorem}

\begin{proof}
    Recall that $\mathcal{X}_{tr}$ denotes the training split containing $N_{tr}$ samples and 
    that $\mathcal{X}_{val}$ denotes the validation split containing $N_{val}$ samples.
    For any $\mathbf{x}^{(i)} \in \mathcal{X}_{val} \cup \mathcal{X}_{ts}$, according to Algorithm 1, 
    we obtain its corresponding $\mathbf{h}^{(i)}$ and $\mathbf{p}^{(i)}$. 
    Note that $\mathbf{W}_h$ exhibits a block diagonal structure and $\mathbf{W}_p = I_{C \times C}$, we can have $\mathbf{W}_p \mathbf{p}^{(i)} =  \mathbf{p}^{(i)}$ and 
    $\mathbf{W}_h \mathbf{h}^{(i)} = \frac{1}{K} [\sum_k s_{1k}^{(i)}, ..., \sum_k s_{Ck}^{(i)}]^T := \mathbf{s}^{(i)}$. 
    Then, 
    the updated node feature $\mathbf{z}^{(i)'}$ can be decomposed into 
    two additive terms:
    \begin{equation}
        \begin{split}
            \mathbf{z}^{(i)'} 
                 = \left[\mathbf{W}_h \mathbf{h}^{(i)} \mathbin\Vert \mathbf{W}_p \mathbf{p}^{(i)}\right] 
                 =  \begin{bmatrix}
                        \boldsymbol{0} \\
                        \mathbf{p}^{(i)}
                        \end{bmatrix} + 
                    \begin{bmatrix}
                        \mathbf{s}^{(i)} \\
                        \boldsymbol{0}
                    \end{bmatrix}
        \end{split}
    \end{equation}
    
    Now define the original node feature $\mathbf{z}^{(i)} := [\mathbf{0} \mathbin\Vert \mathbf{p}^{(i)}] \in \mathbb{R}^{2C}$ for $\mathbf{x}^{(i)} \in \mathcal{X}_{val} \cup \mathcal{X}_{ts}$. Then, the updated node feature $\mathbf{z}^{(i)'}$ can be written into a linear combination of $\mathbf{z}^{(i)}$ and all its neighbor node features:
    \begin{equation}
        \label{equ:nodefeat}
        \begin{split}
            \mathbf{z}^{(i)'} 
                &=  \mathbf{z}^{(i)} + \sum_{c=1}^C \frac{1}{K} \sum_k {s_{ck}^{(i)}} \cdot [\mathbf{e}_c \mathbin\Vert \boldsymbol{0}] \\
                &=  \mathbf{z}^{(i)} + \sum_{\mathbf{n} \in \mathcal{N}^{(i)}} \sum_{c=1}^C {\mathbb{I}\{\mathbf{n} \in \mathcal{X}_c\}} \frac{\mathcal{K}_{f}(\mathbf{x}^{(i)}, \mathbf{n})}{K} \cdot [\mathbf{e}_c \mathbin\Vert \boldsymbol{0}]
        \end{split}
    \end{equation}
    where $\mathcal{N}^{(i)}$ denotes the neighborhood of the sample $\mathbf{x}^{(i)}$ that 
    contains class-wise $K$-nearest neighbors of $\mathbf{x}^{(i)}$ and 
    $\{\mathbf{e}_c | c = 1, ..., C\}$ denotes the standard basis for $\mathbb{R}^C$.
    Then, using matrix notation, we recognize Eq.~\eqref{equ:nodefeat} as 
    \begin{equation}
    \label{eq:gcn}
    \mathbf{Z}' =  \mathbf{Z} + \mathbf{Z}\mathbf{A} = \mathbf{Z}( I + \mathbf{A}),
    \end{equation}
    where $\mathbf{Z}  \in \mathbb{R}^{2C \times (N_{tr} + N_{val})} $ is given by
    \begin{equation}
    \mathbf{Z} = 
        \begin{bmatrix}
            | & \cdots & | & | & \cdots & | \\
            \mathbf{y}^{(1)} & \cdots & \mathbf{y}^{(N_{tr})} & \boldsymbol{0} & \cdots & \boldsymbol{0} \\
            | & \cdots & | & | & \cdots & | \\
            | & \cdots & | & | & \cdots & | \\
            \boldsymbol{0} & \cdots & \boldsymbol{0} & \mathbf{p}^{(1)} & \cdots & \mathbf{p}^{(N_{val})} \\
            | & \cdots & | & | & \cdots & |
        \end{bmatrix}
    \end{equation}
    and $\mathbf{A} \in \mathbb{R}^{(N_{tr} + N_{val}) \times (N_{tr} + N_{val})} $ is given by
    \begin{equation}
    \mathbf{A} = 
        \begin{bmatrix}
            \mathbf{O}_{N_{tr} \times N_{tr}} & \mathbf{A}^{tr\text{-}val} \\
            \mathbf{A}^{val\text{-}tr} & \mathbf{O}_{N_{val} \times N_{val}}
        \end{bmatrix} 
    \end{equation}
    Here, $\mathbf{A}$ is composed of four sub-matrices $\mathbf{O}_{N_{tr} \times N_{tr}}$, $\mathbf{A}^{tr\text{-}val}$, $\mathbf{A}^{val\text{-}tr}$ and $\mathbf{O}_{N_{val} \times N_{val}}$. 
    Among them, $\mathbf{A}^{tr\text{-}val}$, $\mathbf{A}^{val\text{-}tr}$ are affinity matrices 
    that characterize the similarity between each pair of samples (one from $\mathcal{X}_{tr}$ and the other from $\mathcal{X}_{val}$, or the other way around). 
    That is, each entry of $\mathbf{A}$ is given by
    \begin{equation}
    \mathbf{A}_{i,j} =
        \begin{cases}
            0, \quad & \text{if } i \le N_{tr},  j \le N_{tr}; \\
            0, \quad & \text{if } i > N_{tr},  j > N_{tr}; \\
            \frac{\mathcal{K}_f(\mathbf{x}^{(i)}, \mathbf{x}^{(j)})}{K} \cdot \mathbb{I}(\mathbf{x}^{(j)} \in \mathcal{N}^{(i)}), \quad & \text{otherwise}
        \end{cases}
    \end{equation}
    where $\mathcal{N}^{(i)}$ denotes the neighborhood of the sample $\mathbf{x}^{(i)}$ that contains class-wise $K$-nearest neighbors of $\mathbf{x}^{(i)}$.
    
    Therefore, using matrix notation, we obtain a one-hop update Eq.~(\ref{eq:gcn}) in GCNs, 
    followed by nonlinearity and a linear transform, which concludes the proof. Such one-hop GCNs induce a bipartite graph 
    $\mathcal{G} = (L \cup R, E)$, where the nodeset $L$ corresponds to the training samples and $R$ corresponds to the validation samples;  $E$ is the edge set. The adjacency matrix of $\mathcal{G}$ is the matrix $\mathbf{A}$ as defined above.
    The node feature matrix is denoted as $\mathbf{Z}$.
\end{proof}

\section{Proof of Theorem 2}
\label{app:mislabel-proof}


Before the proof, we introduce the following notations.
We use $\{\cdots\}$ to denote a set and $(\cdots)$ to denote an ordered tuple. 
For example, sorting a set $\{3, 1, 4, 2\}$ in decreasing order yields an ordered tuple $(4, 3, 2, 1)$.

\begin{theorem}[\methodcmd Coverage Guarantee]
    \label{app:theo:convergency}
    For any given confidence level $\alpha \in (\frac{1}{N+1}, 1)$,
    with probability at least $1-\alpha$ over the random choice of any correctly labeled data point $(\Tilde{\mathbf{x}}, \Tilde{y})$, we have     $$ \Tilde{r} > \tau_\alpha, $$
    where $\Tilde{r}$ is the predicted reliability score of $\Tilde{\mathbf{x}}$ and 
    $\tau_\alpha$ is defined in \equationautorefname{~(\ref{eq:app-metcmd})}.
\end{theorem}

\begin{proof}
Given a noisy-labeled dataset $\{(\mathbf{x}^{(i)},y^{(i)}) \}_{i=1}^{N}$ with the mislabelling rate $p$ and 
with their corresponding reliability score $\{r^{(i)} \}_{i=1}^{N}$ estimated by \methodcmd, 
we assume that the correctly labeled datapoints obey some distribution $\mathcal{U}$, i.e., 
$(\mathbf{x},y) {\overset{iid}{\sim}} \mathcal{U}$ and the mislabeled datapoints obey 
$\mathcal{V}$, $(\mathbf{x},y) {\overset{iid}{\sim}} \mathcal{V}$. 
Given any correctly labeled test sample $(\Tilde{\mathbf{x}}, \Tilde{y}) \sim \mathcal{U} $, 
Algorithm 2 returns its reliability score $\Tilde{r}$ . 
Then, we sort the reliability scores of the training dataset combined with the test sample $\Tilde{\mathbf{x}}$
in non-increasing order as $(r_{(i)} )_{i=1}^{N+1}$.

Note that there are $M= \left\lfloor Np \right\rfloor$ mislabeled datapoints. 
Let $(r_{(i)}^{\mathcal{U}})_{i=1}^{N-M+1}$ denote the reliability score of correctly labeled samples in non-increasing order. 
According to Algorithm 2, the threshold $\tau_\alpha$ is set to the $B_\alpha$-th largest element, i.e., 
\begin{equation}
    \label{eq:app-metcmd}
    \tau_\alpha = r_{(B_\alpha)}, \text{ where } B_\alpha = \left\lceil(N+1)(1- \alpha)  +\alpha N p \right\rceil.
\end{equation}

Note that $B_\alpha$ should be smaller than or equal to $N$, which yields $\alpha > \frac{1}{N+1}$.

Suppose there exist $m \le M$ mislabeled datapoints whose reliability scores $r^{\mathcal{V}} \ge r_{(B_{\alpha})}$. Then, the probability that the reliability score of the test sample $\Tilde{\mathbf{x}}$ is less than or equal to the threshold $\tau_\alpha$ is given as follows,
\begin{align} \label{eq1}
    \mathbb{P}\left(\Tilde{r} \leq \tau_\alpha \right) 
    &= \mathbb{P}\left(\Tilde{r} \leq r_{(B_\alpha)} \right) \\
    &=  \mathbb{P}\left(\Tilde{r} \leq r_{(B_\alpha-m)}^{\mathcal{U}} \right) \\
    & \leq \frac{N+1-m - (B_\alpha-m)}{N+1-m} \\
    & = \frac{N+1- B_\alpha}{N+1-m}  \\
    & \leq \frac{N+1-B_\alpha}{N+1-M}  \\
    &=  \frac{N+1- \left\lceil(N+1)(1- \alpha)  +\alpha N p  \right\rceil }{N+1-M} \\
    &\leq \frac{N+1- (N+1)(1- \alpha)  - \alpha N p  }{N+1-M} \\
    & = \frac{(N+1)\alpha  - \alpha N p  }{N+1-M} \\
    & \leq \frac{(N+1)\alpha  - \alpha M  }{N+1-M} \\
    & = \alpha
\end{align}
Therefore,
\begin{equation}\label{eq2}
\begin{split}
    \mathbb{P}\left(\Tilde{r} > r_{(B_\alpha)} \right) \geq 1 - \alpha.
\end{split}  
\end{equation}

Note that the inequality (A.11) comes from the fact that $\mathbf{x}$ is an i.i.d. sample
with the other correctly labeled datapoints, i.e. those from $\mathcal{U}$~\citep{balasubramanian2014conformal}.
\end{proof}

\section{Mechanism of \method}
\label{app:mechanism}

\begin{figure}[!t]
    \centering    
    \includegraphics[width=1.0\linewidth]{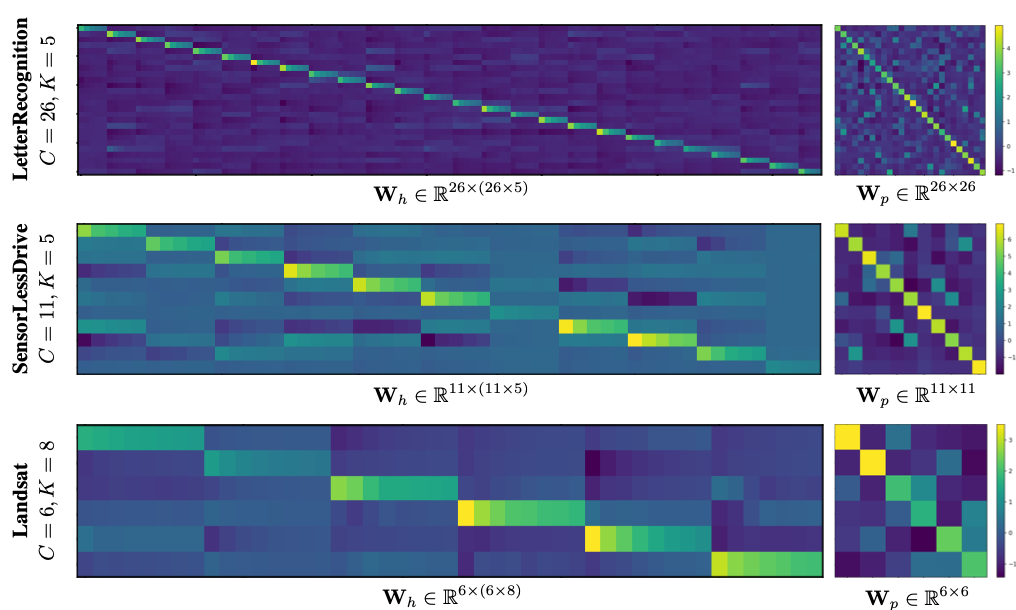}
    \caption{\textbf{Visualization of weight matrices $\mathbf{W}_h$ and $\mathbf{W}_p$ learned by \method}. Brighter colors indicate larger values. $\mathbf{W}_h$ and $\mathbf{W}_p$ indeed exhibit (blockwise) diagonal structures, which verifies the hypothesis of Theorem 1 up to multiplicative factors that can be absorbed by $\mathbf{W}$.}
    \label{fig:W1W2} 
\end{figure}

\method is built upon the idea that the trustworthiness of a classifier's prediction for a sample are highly dependent on two factors: the neighborhood of the sample and the predictive output of the classifier. Therefore, we constructed a model that leverages these two factors for trustworthiness prediction. But how do these two pieces of information interact with each other in determining the trustworthiness score? Curious about this question, we investigate the mechanism empirically. The empirical studies show that the two factors act in a complementary manner: when the sample’s neighborhood information is not accurate or useful, it relies more on the classifier output; and vice versa.
We confirm this point via closer scrutiny on \figureautorefname{~\ref{fig:W1W2}}, where those brighter diagonal blocks in $\mathbf{W}_h$ are empirically associated with their corresponding darker diagonal entries in $\mathbf{W}_p$, and vice versa. To further confirm it quantitatively, we calculate the Pearson correlation coefficient $\rho$ between the mean of the diagonal blocks in $\mathbf{W}_h$ and the activation of their corresponding diagonal entries in $\mathbf{W}_p$. The results show a relatively strong negative correlation (e.g. $\rho=-0.90$ on Landsat), which sheds light on the complementary mechanism of \method. 

\begin{figure}[!t]
    \centering    
    \includegraphics[width=0.7\linewidth]{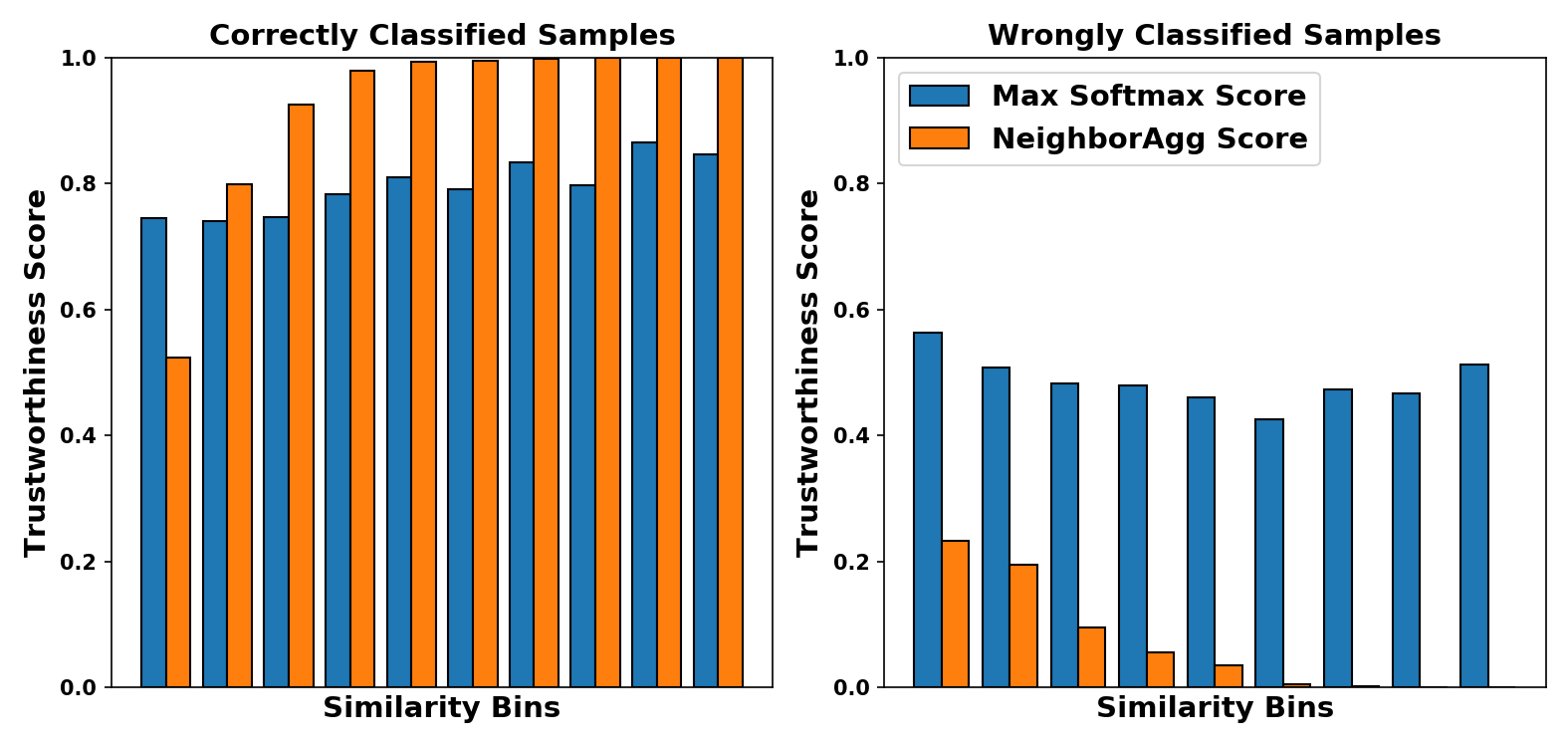}
    \caption{\textbf{Visualization of softmax scores and trustworthiness scores} in different similarity bins (i.e., local density). Samples are split into two categories: correctly classified and incorrectly classified. For every sample $\mathbf{x}$, we take average similarity to its K neighbors of predicted class to characterize the intra-class local density.}
    \label{fig:intuition-app} 
\end{figure}

In what follows, we empirically show how \method adjusts the confidence score of a given base classifier in a complementary manner through another perspective.

There are two cases of interest: 
\begin{compactitem}
    \item how \method reacts when a sample is correctly classified; 
    \item how \method reacts when a sample is misclassified.
\end{compactitem}

The experiment is carried out on the test set. For each sample $\mathbf{x}$ in the set (with the associated label y), we find its intra-class K nearest neighbors $n_{yk}$ ($k=1,...,K$), and compute the corresponding similarity $s_{ck}$. For each $\mathbf{x}$, we further take the mean of these similarities ${s_{ck}, k=1,...,K}$ as a measure that characterizes the intra-class local density of $\mathbf{x}$ in the feature space. On the other hand, we can obtain the sample's confidence score given by the base classifier and the trustworthiness score given by our \method, respectively. Then we plot the relation between the scores and the local density in \figureautorefname{~\ref{fig:intuition-app}}. 

Left-hand chart in \figureautorefname{~\ref{fig:intuition-app}} illustrates the first case where a sample is correctly classified. It reveals that when the sample resides in a lower density region (i.e. neighbors are scattered distantly and hence have low mean similarity values), our \method can adaptively lower the original confidence score by putting more weights on the neighborhood. When the sample resides in a higher density region (i.e. neighbors are concentrated locally and hence have high mean similarity values), our \method will increase the original confidence score by exploiting more information from the classifier output. In the second case (see right-hand chart of \figureautorefname{~\ref{fig:intuition-app}}) where a sample is misclassified, \method generally lowers the score regardless of the local density of $\mathbf{x}$ but to varying degrees. For example, in higher density regions, \method tends to dramatically lower the score than in lower density regions. Generally, \method can enhance the scores when the samples are correctly classified and when its neighbors are close, and can decrease the scores otherwise. These phenomena consistently agree with our proposed motivation.


\begin{figure}[t]
    \centering    
    \subfigure[APC vs. $K$]{\includegraphics[width=0.47\linewidth]{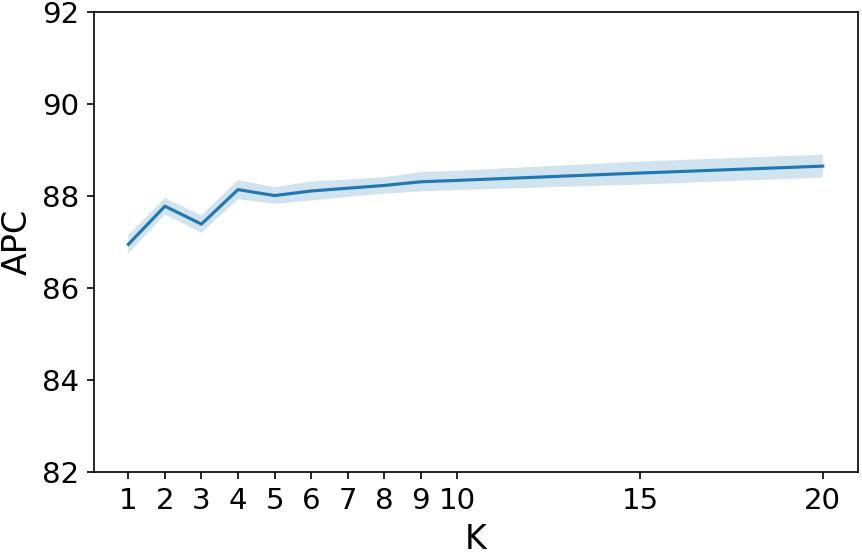}} \,
    \subfigure[APM vs. $K$]{\includegraphics[width=0.47\linewidth]{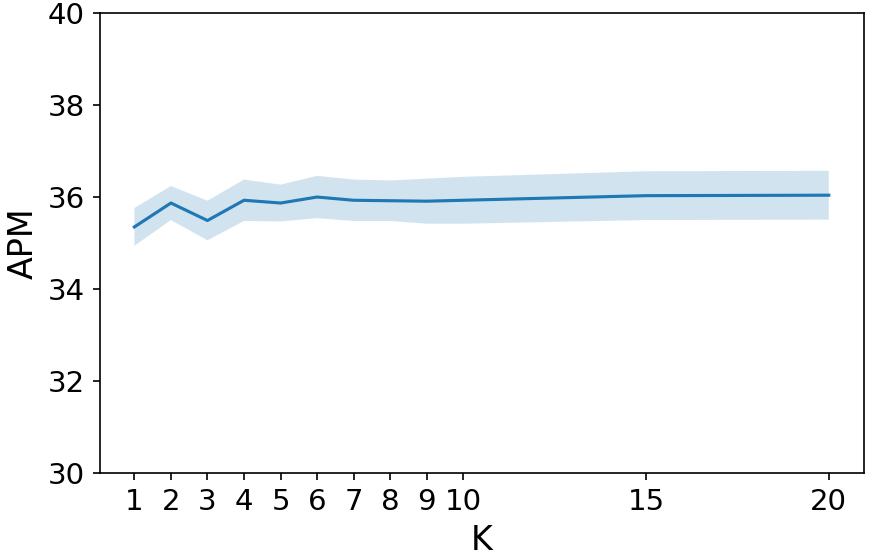}}
    \vspace{-0.1in}
    \caption{\textbf{Performance vs. \# of neighbors $K$ for \method}. 
         Our \method is insensitive to $K$ 
         within a wide range under different metrics. Experiments are performed on CardDefault using LR as the base classifier.}
    \label{fig:hyperpar_sen} 
\end{figure}

\section{Sensitivity to Hyperparameters}
\label{app:hyperparameter-sensitive}

We investigate the sensitivity to the hyperparameter $K$, i.e. the number of neighbors used in the neighborhood vector. We set $K$ from 1 to 20 with other parameters fixed. \figurename{~\ref{fig:hyperpar_sen}} shows the performance under APM and APC
over the range of $K$ with the shaded area representing its standard deviation. 
We observe that within the verified scope, more neighbors lead to better performance and when $K$ reaches a certain number the performance curve becomes smooth.
The result verifies that our proposed method is hyperparameter-insensitive when a moderate $K$ is chosen because it can automatically learn the intra-class topology. As more neighbors result in higher computational overheads,
we generally suggest to set $K = 5$ for tabular datasets and $K=10$ for image datasets to trade computational cost off against performance.





\section{Computational and Memory Cost}
\label{app:comp-cost}

This section discusses the computational and memory costs of inference time for our proposed method and other comparative approaches, including analytical time and space complexity as well as actual time consumption. 

\paragraph{Time Complexity}
We start by giving an intuitive example of the specific elapsed time of our method and then analyze how our method scales with \# of classes, dataset size and model size.
Empirically, taking CIFAR10 dataset as an example, we provide the inference time of our method for 10000 samples on 1 Tesla V100 SXM2 GPU with 32GB memory and Intel Xeon CPU E5-2698 v4 @ 2.20GHz with 504GB available memory. The neighbor search is executed on either GPU (faiss) or CPU (scikit-learn). To make the time cost more intuitive, we also provide the elapsed time of ResNet18 classifying 10000 samples as a baseline comparison. The results shows that GPU-accelerated neighbor search can be 3 times faster than the CPU version, and the CPU version spends comparable time with a common classification network.

\begin{table}[h]
    \centering
    \caption{
    Running time comparison for $10,000$ queries on GPU and CPU for CIFAR10 dataset.
    }
    \begin{tabular}{l|c|c|c} \toprule
                    & GPU (min)          & CPU (min)  & Baseline     \\  \hline
        10000 queries & 0.238 & 0.768 & 0.790 \\ 
        \bottomrule
    \end{tabular}
\label{tab:computation}
\end{table}

Analytically, the inference time complexity of our proposed method is determined by two components: neighbor search and aggregation, with the former shared solely by methods requiring neighborhood information, such as Trust Score and ours.

First, the aggregation operator is instantiated as a neural network, whose time complexity is determined by the number of layers of the network. This is similar to most methods in this field, such as TCP, suggesting that our method is similar to the others in the worst case. However, our aggregator is usually faster since it has fewer layers (e.g., one layer in our implementation), which is because it only deals with low-dimensional auxiliary information (e.g., similarities to neighbors) rather than high-dimensional feature embedding. In this sense, the aggregation process does not incur high computational cost.

Second, without placing a high priority on computational cost, we have chosen the simple and widely-used kd-trees from scikit-learn library for nearest neighbor search. It is worth mentioning that neighbor search is a well-studied problem for which numerous algorithms have been developed with different benefits. In terms of time complexity, there are many algorithms aiming to reduce the computing time, such as locality-sensitive hashing and randomized kd-trees~\cite{andoni2017nearest}, etc. 

For our implementation, the average time complexity of the neighbor search phase is $O(Kd \cdot \sum_{c=1}^C \log N_c)$ where $K, C, d$ are the \# of neighbors, classes and embedding dimension, and $N_c$ denotes the size of the kd-tree for class $c$ (the \# of samples for class $c$ in the training dataset).

\begin{itemize}
    \item First, a small value of $K$ is sufficient for most cases (e.g., $K = 20$ in our implementation) and it does not increase too much with $N$. Besides, because the kd-tree organizes similar samples together, after a single query we can simply search the surroundings around that point to find other nearest neighbors. Therefore, the implementation can be optimized by querying only once.
    \item Secondly, the size of dataset $N$, has limited impact due to the logarithmic scaling. For example, as $N$ increases from $0.05$M (CIFAR10) to $1.20$M (ImageNet), the query time only increases by around 1.3 times, which is fairly insignificant. 
    \item Thirdly, even though it grows linearly with $C$, each class's associated kd-tree can be executed in parallel to mitigate the effect, especially for tasks with a large value of $C$, such as $C=1,000$ in ImageNet, since they are independent of each other. Besides, the computational power of GPU and other hardware acceleration techniques can also be utilized, as will be demonstrated below.
\end{itemize}

Regarding model size, it does not affect scalability because the input of our technique is the classifier's output, which has been fixed to dimension $C$.

Therefore, we argue that computational cost is unlikely to become the bottleneck for the deployment of our approach.

\paragraph{Space complexity}

The memory cost of our method can also be decomposed into neighbor search and aggregation cost. As we have illustrated before, the aggregator typically has fewer layers and smaller input dimensions, hence incurring little memory cost. As for neighbor search, the memory cost is dependent on the size of all kd-trees $N$ and the feature dimension $d$, with an analytical space complexity of $O(Nd)$. This analysis applies to all methods requiring neighbor search, including Trust Score. Typically, the feature dimension $d$ for complex data, such as images, is approximately proportional to dimension of manifold, which is substantially lower than the original data dimension. For example, ImageNet with feature dimension $d=256$ consumes less than $300$MB of memory, which can be deployed on CPU or GPU. Moreover, memory-efficient kd-tree~\cite{choi2013improving,rafiee2019pruned} and training dataset compression techniques can also be used for reducing memory cost further.

\section{Graph Neural Network Framework}
\label{app:gnn-algo} 

Graph neural networks \cite{kipf2016semi, xu2018powerful} have been a topic of interest in recent times for their powerful modeling ability in aggregating information from different nodes. It also comes as a natural option for aggregating neighborhood information in our trustworthiness framework under our neighbor-homophily assumption. This motivates us to compare our method with GNNs.

In order to show the similarity and difference between our \method and GNNs, we also provide a detailed description 
about the framework of using GNNs for trustworthiness prediction, in particular, the process of building graph, training and inference.

Firstly, we construct a distance-weighted graph by instantiating its adjacent matrix. Assume that we have the feature matrix $X_{tr} \in R^{N \times d}$  of the training dataset where $N$ is the number of samples and $d$ is the feature dimension. We compute its pairwise distance based on certain similarity metrics. Different type of datasets are suitable for different similarity kernels and the most common one is Minkowski distance $d(\mathbf{A}, \mathbf{B})=(\sum_{i=1}^{n}\left(A_{i}-B_{i}\right)^{p})^{1/p}$ with $p=2$ (Euclidean distance). Then for every sample, we compute its $K$ nearest neighbors based on Euclidean distance and set edges between them. The weight of every edge is demonstrated by $e^{\frac{-d}{T}}$ where $d$ is the distance between two nodes and $T$ is the temperature parameter. 

Secondly, we construct every sample's node feature. For training dataset, the node feature matrix is $Y_{tr} \in R^{N \times C}$. In order to leverage the classifier outputs, we use softmax outputs of the base classifier as the node feature for the validation set and test set. It is worth noting that the test set's node feature can be constructed after the model has been trained. 

Thirdly, we choose suitable graph neural networks for aggregating these neighbors. Any graph neural networks that are suitable for inductive learning setting can be applied to this task. Without loss of generality, we consider the widely used graph convolutional neural network (GCN) as the backbone for neighborhood aggregation. 

The training objective is kept the same with our \method as shown in \equationautorefname{~\ref{eq:loss}}.

\section{Detailed Experiment Setup}
\label{app:exp-setup}

\paragraph{Tabular Dataset} Following \citet{jiang2018trust}, we split each dataset into $40\%$ for training, $10\%$ for validation and $50\%$ for test. 
Experiments are conducted based on three base classifiers, including logistic regression (LR)~\cite{peng2002introduction}, random forest (RF)~\cite{svetnik2003random} and multi-layer perceptrons (MLPs)~\cite{ruck1990multilayer}.  For RF, we use the default parameter setting suggested in the official implementation~\cite{sklearn_api}, while for LR, we set its maximum iterations to 5000. 
The MLP consists of two hidden layers (200, 70) with ReLU activation and Batch Normalization~\cite{ioffe2015batch}. 
The transform $f$ of the similarity kernel $\mathcal{K}_{f}$ is set to be identity mapping. Empirically, we find that learning with identity mapping yields sufficiently good performance. A more complex transform (e.g. a pretrained network) can be used to achieve higher performance. 
Models in comparison are implemented using their officially open-sourced codes.

\paragraph{Image Dataset}  For the image dataset, 20\% training data points are sampled using random seed 777 as the validation set. For FashionMNIST and MNIST datasets, small convolutional neural networks with two convolutional layers and two dense layers are used as base classifiers, while for CIFAR10, ResNet50 (the official implementation of Pytorch) is used. All these architectures are provided in the code. Similarity is computed using Euclidean distance of features given by the backbone of the base classifier. For the comparison of baseline models, we use their official implementation provided in their papers. Hyperparameters are kept the same across all comparison models.

For neighbor search, \kdtree~\citep{scheuermann1982multidimensional} and Faiss~\citep{JDH17} on CPU/GPU settings has been provided to efficiently search for neighbors of a sample in the large dataset.

\begin{table}[t]
    \small
    \centering 
    \begin{tabular}{l|c|c|l} \toprule 
    \textbf{Datasets}  & \textbf{\#Attributes} & \textbf{\#Classes} & \textbf{\#Instance}  \\ \hline 
         Adult                & 14             & 2              & $48,842$           \\ \hline 
         BankMarketing        & 63             & 2              & $45,211$           \\ \hline 
         CardDefault          & 23             & 2              & $30,000$           \\ \hline 
         Landsat              & 36             & 6              & $64,35$            \\ \hline 
         LetterRecognition    & 16             & 26             & $20,000$           \\ \hline 
         MagicGamma           & 10             & 2              & $19,020$         \\ \hline 
         SensorLessDrive      & 48             & 11             & $58,509$          
        \\ \bottomrule 
    \end{tabular}
    \caption{Dataset Statistics.}
    \label{tab:data_info} 
\end{table}

\section{More Results on Ablation Study}
\label{app:ablation-app}

In this section, we show extensive ablation study results in Figure~\ref{fig:ablation-app} to demonstrate the effectiveness of each component and 
the adaptiveness of the \emph{learnable} weights in \method. 

We compare against its variants \predonly and \neionly, as well as Trust Score.
\begin{compactitem}
    \item{\neionly:} solely takes neighborhood vectors as input,
    \item{\predonly:} solely inputs classifier output vectors.
    
\end{compactitem} 


\figurename{~\ref{fig:ablation}} demonstrates the comparison results of ablation study on seven tabular datasets across every base classifiers and every metrics. We note that our \method (represented by leftmost orange bars of every group) consistently outperforms \predonly and \neionly across almost all datasets, indicating that both vectors make non-negligible contributions to the final \trustscore,  
and that considering either of them alone is insufficient. 
It supports the working mechansim that the neighborhood and the predictive output of the classifier complement each other in determining the \trustscore.

\begin{figure*}[thb]
    \centering    
    \subfigure[Comparison for LR base classifiers under the AUC metric.\label{app:fig:abl_1}]{
         \includegraphics[width=0.48\linewidth]{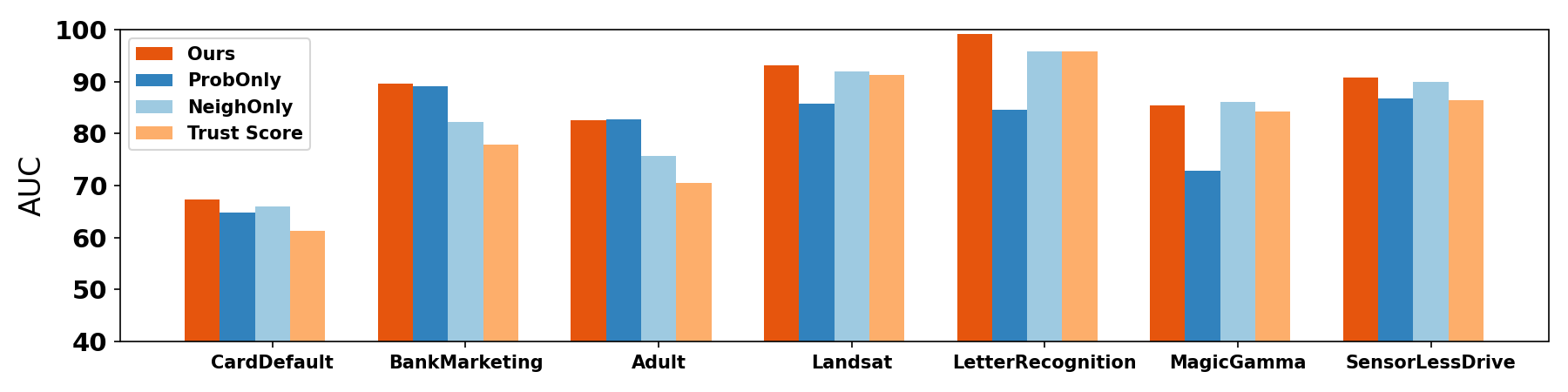} }  \,
    \subfigure[Comparison for LR base classifiers under the APC metric. \label{app:fig:abl_2}]{
         \includegraphics[width=0.48\linewidth]{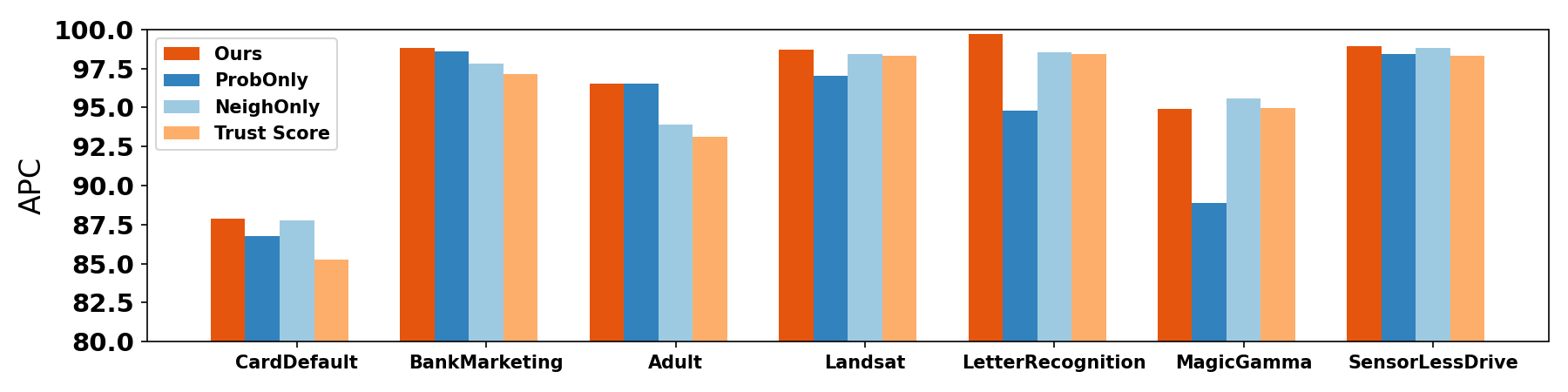} }  \\
         
    \subfigure[Comparison for RF base classifiers under the APC metric.\label{app:fig:abl_1}]{
         \includegraphics[width=0.48\linewidth]{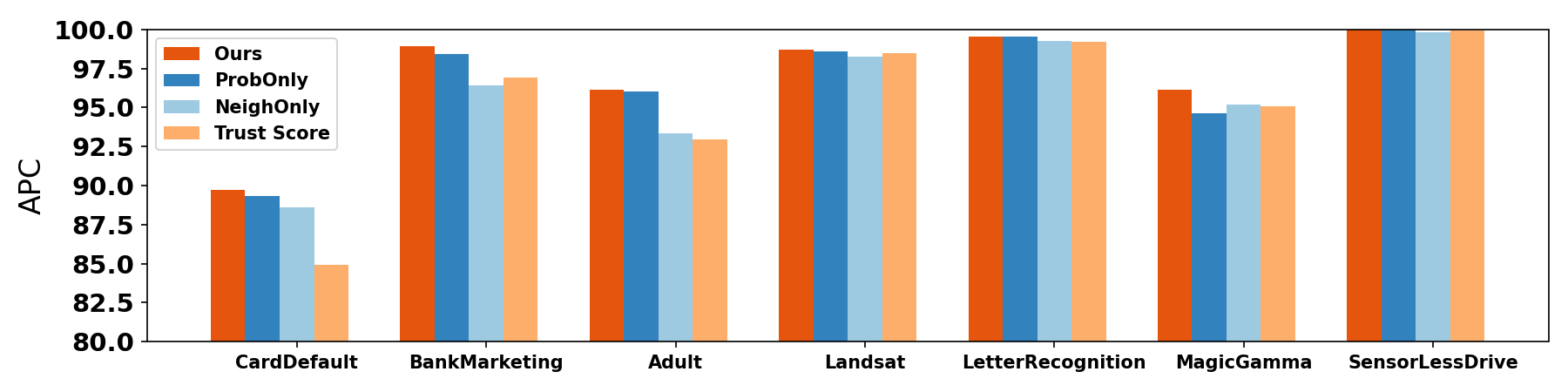} }  \,
    \subfigure[Comparison for RF base classifiers under the APM metric.\label{app:fig:abl_1}]{
         \includegraphics[width=0.48\linewidth]{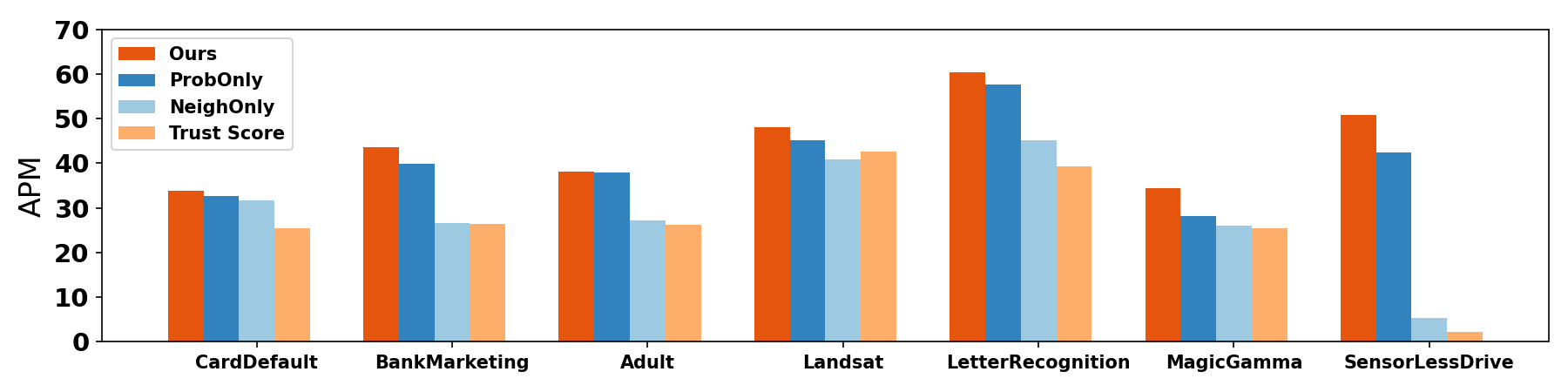} } \\
         
    \subfigure[Comparison for MLP base classifiers under the AUC metric.\label{app:fig:abl_1}]{
         \includegraphics[width=0.48\linewidth]{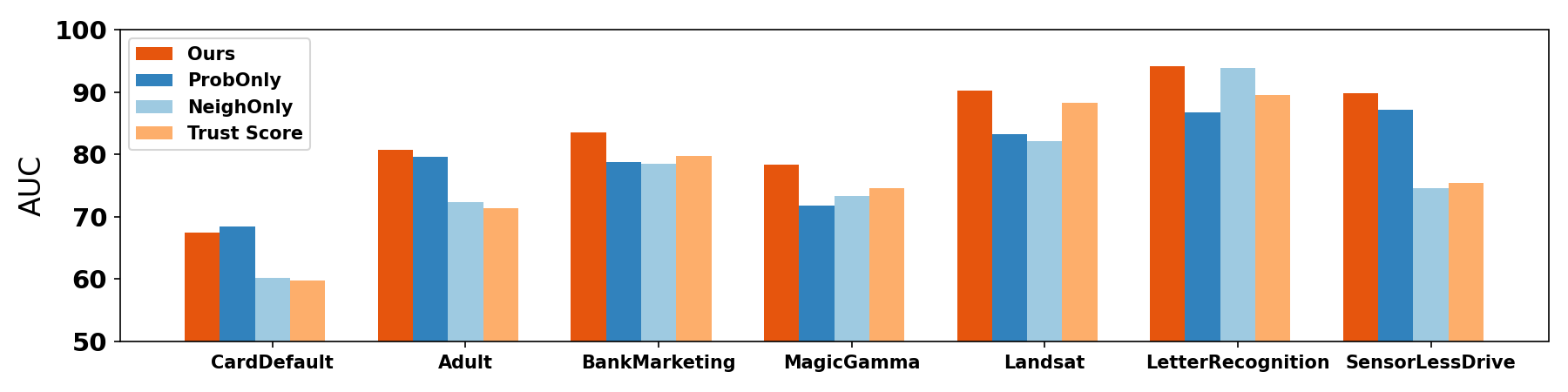} }  \,
    \subfigure[Comparison for MLP base classifiers under the APC metric.\label{app:fig:abl_1}]{
         \includegraphics[width=0.48\linewidth]{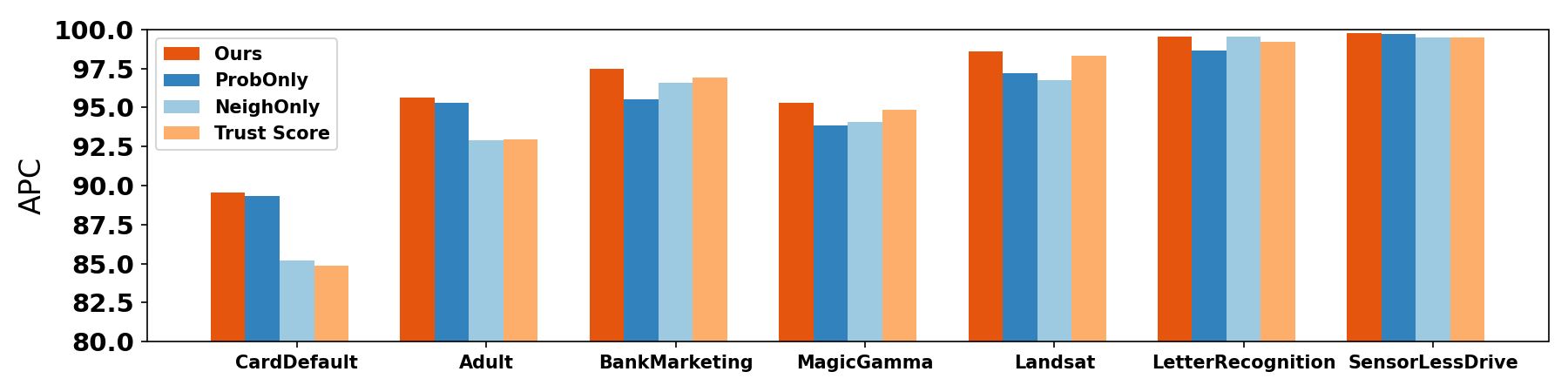} } \\
 
    \subfigure[Comparison for MLP base classifiers under the APM metric.\label{app:fig:abl_1}]{
         \includegraphics[width=0.48\linewidth]{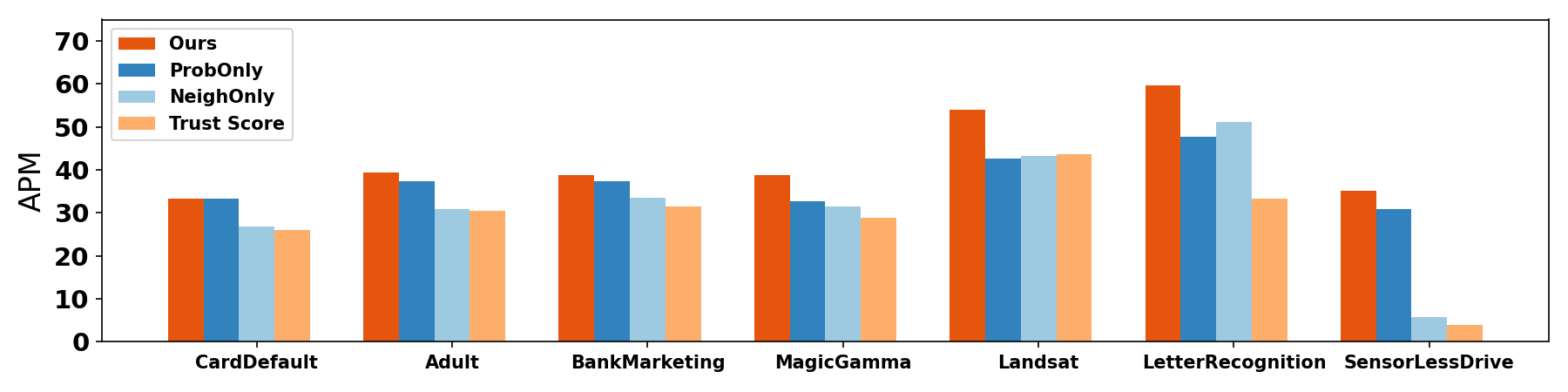} }         
         
    \caption{\textbf{Performance results for ablation study.} 
      }
    \label{fig:ablation-app}
\end{figure*}

\section{Mislabel Detection: A Case Study}
\label{app:mislabel}

This section showcases the effectiveness of our \methodcmd algorithm. 
We use a dataset named \textit{QuoraInsQ} in the Kaggle competition~\cite{QuoraInsQ:Kaggle} to identify and flag insincere questions.
This dataset consists of 1,306,122 questions.

Firstly, we select the model of the top-ranking team from the leader board as our base classifier 
and use \methodcmd to compute the reliability score for each sample.
To determine the mislabelling rate, 
we manually relabel the data and estimate $p$ as $0.03$.
Then, we run \methodcmd with the confidence level $\alpha=5\%$ and obtain the detected mislabeled results. 

\tablename{~\ref{tab:mislabel}} showcases 100 detected example questions 
with the lowest trustworthiness scores.
We find that all of them are mislabeled as \emph{`insincere'} in the original dataset. 

\begin{longtable}{|p{0.90\linewidth}|p{0.05\linewidth}|}
\toprule
\textbf{Detected mislabeled example questions that are originally labeled as insincere}  & \textbf{Score}  \\ \hline
\endfirsthead
\toprule
\textbf{Detected mislabeled example questions that are originally labeled as insincere}  & \textbf{Score}  \\ \hline
\endhead
What is the difference between a fusion and a restaurant?& 0.0207 \\ \hline
What are the new product for agent project? & 0.0212 \\ \hline
Which protagonist from a video game have you most related to?                                                                                                                                                                                 & 0.0216 \\ \hline
Do I have to appear for IMU CET again   even if I get a good rank in it if I'm appearing for improvement of HSC board   exam?                                                                                                                   & 0.0216 \\ \hline
What are some important things/steps when   starting a film production company in Netherlands?                                                                                                                                                  & 0.0217 \\ \hline
What astrological combinations are needed   to obtain a scholarship for studies?                                                                                                                                                                & 0.0219 \\ \hline
What advice would you give a person   intending to buy a Nissan note, in terms of performance i.e. traction, fuel   economy, maintenance and resale?                                                                                            & 0.0222 \\ \hline
What is the maximum speed that Balaji   Vishwanathan can achieve on a treadmill while singing bhajans?                                                                                                                                          & 0.0222 \\ \hline
Is Edward Snowden a computer simulated   character produced at Pixar Studios in Hollywood?                                                                                                                                                      & 0.0222 \\ \hline
What Elements Green Relief Have?                                                                                                                                                                                                                & 0.0223 \\ \hline
What kind of relationship is between   Sushi and Happiness?                                                                                                                                                                                     & 0.0223 \\ \hline
How cool are diabetes?                                                                                                                                                                                                                          & 0.0223 \\ \hline
What is similar between sanitary pads and   Subhash Chandra Bose?                                                                                                                                                                               & 0.0223 \\ \hline
Is there any spot round in COMEDK? If   yes, how is it conducting?                                                                                                                                                                              & 0.0224 \\ \hline
What are the more steps in Career   Oriented Education?                                                                                                                                                                                         & 0.0224 \\ \hline
Which branches can I get in cse in UPES   at the rank of 20000?                                                                                                                                                                                 & 0.0224 \\ \hline
How can I become Bahubali?                                                                                                                                                                                                                      & 0.0225 \\ \hline
What are some ways I can help melt the   polar ice caps?                                                                                                                                                                                        & 0.0226 \\ \hline
What are the possible motives of RAW for   assassination of legend Bollywood actor OM Puri?                                                                                                                                                     & 0.0228 \\ \hline
Are there any aquaculture interest-free   loans?                                                                                                                                                                                                & 0.0228 \\ \hline
Which good ingredients include in hair   Eternity?                                                                                                                                                                                              & 0.0228 \\ \hline
What's the function of Charcoal mask   blackhead remover mask?                                                                                                                                                                                  & 0.0228 \\ \hline
Which uses more fuel, air-conditioning on   or Windows and the sun roof open?                                                                                                                                                                   & 0.0229 \\ \hline
When we park a car in the Sun, what is   the phenomenon known as?                                                                                                                                                                               & 0.0229 \\ \hline
Does the CTMU say that reality consists   of mutual perceptions of various operators which have internal or cognitive,   and "external" or informational aspects? Do atoms recognize and   process each other according to the laws of physics? & 0.0229 \\ \hline
Is frozen tamod a popular dessert in the   Philippines?                                                                                                                                                                                         & 0.0229 \\ \hline
How do you unlock your phone from a Gmail   account?                                                                                                                                                                                            & 0.023  \\ \hline
What is the function of Hydrapharm   Anatabloc?                                                                                                                                                                                                 & 0.023  \\ \hline
Which is the best Federal skilled worker   visa consultancy for Canada?                                                                                                                                                                         & 0.023  \\ \hline
What is the 3rd letter of the English   alphabet?                                                                                                                                                                                               & 0.023  \\ \hline
What are some bad habits of Indian   recruiter/HR?                                                                                                                                                                                              & 0.023  \\ \hline
What are the requirements to study and   work in Canada for international students?                                                                                                                                                             & 0.023  \\ \hline
What are the best companies for Ruby on   Rails development in Toronto?                                                                                                                                                                         & 0.023  \\ \hline
Is it better to read a book from back to   front or from front to back?                                                                                                                                                                         & 0.0231 \\ \hline
What is freight forwarder management   system?                                                                                                                                                                                                  & 0.0231 \\ \hline
How do I start Canada PR Visa process?                                                                                                                                                                                                          & 0.0231 \\ \hline
Is bullying among the funniest activities   for kids in a traditional classroom?                                                                                                                                                                & 0.0232 \\ \hline
How can I use law of attraction for   passing a exam?                                                                                                                                                                                           & 0.0232 \\ \hline
Are cyclists required to stop during red   lights at Traffic stop?                                                                                                                                                                              & 0.0232 \\ \hline
What is more important to become a pilot   computer or Punjabi?                                                                                                                                                                                 & 0.0232 \\ \hline
How do you get a miracle if the event is   impossible such as walking through a solid wall or flying to the moon?                                                                                                                               & 0.0232 \\ \hline
How can I get certified in rolfing in New   York?                                                                                                                                                                                               & 0.0232 \\ \hline
What is the Amazon exclusive (10.or)   Tenor smartphone company all about?                                                                                                                                                                      & 0.0233 \\ \hline
What is the best detoxification treatment   center in New Jersey?                                                                                                                                                                               & 0.0233 \\ \hline
What are some arguments why an   inappropriate relationship between a teacher and an adult student would be   beneficial for the educational process in the long term?                                                                          & 0.0233 \\ \hline
What is a stock market and how do I club   seals effectively?                                                                                                                                                                                   & 0.0234 \\ \hline
How did I become famous?                                                                                                                                                                                                                        & 0.0235 \\ \hline
What differences are there between the   economic system of a business and that of a country?                                                                                                                                                   & 0.0235 \\ \hline
Where can I supply (KCN) potassium   Cyanide?                                                                                                                                                                                                   & 0.0235 \\ \hline
What industries have the most potential   for now or the near future?                                                                                                                                                                           & 0.0235 \\ \hline
Who is a better wedding planner in Bhopal   and Indore than me?                                                                                                                                                                                 & 0.0235 \\ \hline
What are the easiest ways to find   investment opportunities?                                                                                                                                                                                   & 0.0236 \\ \hline
How or who maintained the harmony in the   complicated system of the extra terrestrial bodies?                                                                                                                                                  & 0.0236 \\ \hline
How can I locate someone who hacked my   Instagram account?                                                                                                                                                                                     & 0.0236 \\ \hline
What did you think of the Sword Art   Online movie?                                                                                                                                                                                             & 0.0236 \\ \hline
Who offers the cheapest car insurance in   New Jersey?                                                                                                                                                                                          & 0.0237 \\ \hline
Is it possible to watch Baahubali in a   theatre without updating your status on Facebook?                                                                                                                                                      & 0.0237 \\ \hline
Who offers the most affordable car   insurance quotes in New Jersey?                                                                                                                                                                            & 0.0237 \\ \hline
What is no pen wall painting?                                                                                                                                                                                                                   & 0.0237 \\ \hline
Can we have the patent or copyright for a   website?                                                                                                                                                                                            & 0.0237 \\ \hline
What are the basics of auto insurance?                                                                                                                                                                                                          & 0.0237 \\ \hline
How can I achieve divine status?                                                                                                                                                                                                                & 0.0238 \\ \hline
Under what circumstance would Tom Daley   become king?                                                                                                                                                                                          & 0.0238 \\ \hline
What do you look like in a revealing   dress?                                                                                                                                                                                                   & 0.0238 \\ \hline
How do I track someone's location without   having to add anything on their phone?                                                                                                                                                              & 0.0238 \\ \hline
How do psychedelic drugs affect   introverts?                                                                                                                                                                                                   & 0.0238 \\ \hline
What are the new markets and investment   opportunities this year?                                                                                                                                                                              & 0.0238 \\ \hline
Which are the best consultancies in   Chandigarh?                                                                                                                                                                                               & 0.0238 \\ \hline
What anomalies do surgeons notice between   different races?                                                                                                                                                                                    & 0.0238 \\ \hline
Should I still work as a software   developer if I earn 40 million USD by selling my company's ICO token?                                                                                                                                       & 0.0238 \\ \hline
What is the Trumputin effect?                                                                                                                                                                                                                   & 0.0239 \\ \hline
What could be the reason for people   buying clothes online instead of the retail shops which has a good quality of   clothes but the only drawback could be the price differences?                                                             & 0.0239 \\ \hline
Where can I find android app developers?                                                                                                                                                                                                        & 0.0239 \\ \hline
How can I permanently set priority in the   windows task bar?                                                                                                                                                                                   & 0.0239 \\ \hline
What are the new outstanding investment   opportunities this year?                                                                                                                                                                              & 0.0239 \\ \hline
Are there any games that have similar   economy gameplay that Runescape used to have in its prime?                                                                                                                                              & 0.0239 \\ \hline
How can someone score 180+ in JEE (MAIN)   2017 if he/she becomes serious just one month before the D-Day?                                                                                                                                      & 0.024  \\ \hline
What kinds of crossover SUVs does   Chevrolet offer?                                                                                                                                                                                            & 0.024  \\ \hline
What was the best SUV of the year in   2017?                                                                                                                                                                                                    & 0.0241 \\ \hline
Apex Rush Testo reviews: How Does It   Work? Read Before BUY?                                                                                                                                                                                   & 0.0241 \\ \hline
Why is the burrete readings differences   are equal?                                                                                                                                                                                            & 0.0241 \\ \hline
What are the components of the XL Test   Plus?                                                                                                                                                                                                  & 0.0241 \\ \hline
How do I reset a LinkedIn password?                                                                                                                                                                                                             & 0.0242 \\ \hline
Which is the best Surrogacy Clinic in   Kolkata?                                                                                                                                                                                                & 0.0242 \\ \hline
What are some popular night spots in   Chicago?                                                                                                                                                                                                 & 0.0242 \\ \hline
Where can I find best romantic shayaris?                                                                                                                                                                                                        & 0.0242 \\ \hline
How easy it is for a girl to switch from   one guy to other in a very short span?                                                                                                                                                               & 0.0243 \\ \hline
What's the potential for the decoration   industry?                                                                                                                                                                                             & 0.0243 \\ \hline
How do hackers think and work?                                                                                                                                                                                                                  & 0.0243 \\ \hline
What would be the process of selling The   United States?                                                                                                                                                                                       & 0.0243 \\ \hline
What are the best business to start with   20 to 30 Lacs of investments?                                                                                                                                                                        & 0.0243 \\ \hline
How can I register as a freight forwarder   in international services?                                                                                                                                                                          & 0.0243 \\ \hline
When did the UK switch to English?                                                                                                                                                                                                              & 0.0243 \\ \hline
What are the best thesis writing service   provider in kuwait?                                                                                                                                                                                  & 0.0243 \\ \hline
What is the best site to Download   Avengers: Infinity War in HD?                                                                                                                                                                               & 0.0243 \\ \hline
What is the global best low investment   business?                                                                                                                                                                                              & 0.0244 \\ \hline
Is crossing between spider and humans   possible?                                                                                                                                                                                               & 0.0244 \\ \hline
What kinds of animals do you have to   date? Would you like to add more?                                                                                                                                                                        & 0.0244 \\ \hline
What are the alternatives available to   KCN in Indian pharmacy stores?                                                                                                                                                                         & 0.0244 \\ 
\bottomrule
\caption{ 100 Samples from \textit{QuoraInsQ} that are mislabeled as insincere and detected by \methodcmd and 
    their corresponding reliability scores ($\times 10^2$)}
\label{tab:mislabel-app}
\end{longtable}

\end{document}